\documentclass{article}

\usepackage[final, nonatbib]{neurips_2021}

\usepackage[utf8]{inputenc} 
\usepackage[T1]{fontenc}    

\usepackage{url}            
\usepackage{booktabs}       
\usepackage{amsfonts}       
\usepackage{nicefrac}       
\usepackage{microtype}      
\usepackage{xcolor}         

\usepackage{subfigure}
\usepackage{multirow}
\usepackage{tikz}
\usepackage{bbm}
\usepackage{wrapfig}

\usepackage{algorithm}
\usepackage{algorithmic}
\usepackage{comment}
\usepackage{amsmath}
\usepackage{amssymb}
\usepackage{amsthm}

\newtheorem{definition}{Definition}
\newtheorem*{definition*}{Definition}
\newtheorem*{lemma*}{Lemma}
\newtheorem*{theorem*}{Theorem}
\newtheorem{lemma}{Lemma}
\newtheorem{corollary}{Corollary}
\newtheorem{theorem}{Theorem}

\newcommand{\ours}{CLOvE}

\newcommand{\x}{{\mathbf x}}

\newcommand{\w}{{\mathbf w}}

\newcommand{\M}{{\mathbf{M}}}
\newcommand{\W}{{\mathbf{W}}}

\newcommand{\reals}{\mathbb{R}}

\newcommand{\X}{\mathcal{X}}
\newcommand{\cE}{\mathcal{E}}
\newcommand{\cX}{\mathcal{X}}
\newcommand{\cY}{\mathcal{Y}}

\newcommand{\E}{\mathbb{E}}

\newcommand{\U}{\mathbf{U}}

\newcommand{\N}{\mathcal{N}}

\newcommand{\dsp}{d_{\text{sp}}}

\DeclareMathOperator*{\argmin}{arg\,min} 

\renewcommand{\eqref}[1]{Equation~(\ref{#1})}
\newcommand{\figref}[1]{Figure~\ref{#1}}
\newcommand{\secref}[1]{Section~\ref{#1}}
\newcommand{\thmref}[1]{Theorem ~\ref{#1}}
\newcommand{\lemref}[1]{Lemma~\ref{#1}}

\newcommand{\corref}[1]{Corollary~\ref{#1}}

\newlength\myindent 
\newcommand{\ucomment}[1]{\textcolor{blue}{Uri: #1}}

\def \Xc{X_\text{causal}}
\def \Xac{X_\text{ac-spurious}}
\def \Xacn{X_\text{ac-non-spurious}}

\makeatletter
\newcommand*{\indep}{%
  \mathbin{%
    \mathpalette{\@indep}{}%
  }%
}
\newcommand*{\nindep}{%
  \mathbin{
    \mathpalette{\@indep}{\not}
  }%
}
\newcommand*{\@indep}[2]{%
  \sbox0{$#1\perp\m@th$}
  \sbox2{$#1=$}
  \sbox4{$#1\vcenter{}$}
  \rlap{\copy0}
  \dimen@=\dimexpr\ht2-\ht4-.2pt\relax
  \kern\dimen@
  {#2}%
  \kern\dimen@
  \copy0 
} 
\makeatother

\title{On Calibration and Out-of-domain Generalization}

\author{%
  Yoav Wald\thanks{Equal contribution} \\
  Johns Hopkins University\\
  \texttt{yoav.wald@gmail.com} \\
   \And
  Amir Feder\footnotemark[1] \\
  Technion\\
  \texttt{amirfeder@gmail.com} \\
  \And
  Daniel Greenfeld \\
  Jether Energy Research\\
  \texttt{danielgreenfeld3@gmail.com} \\
  \And
  Uri Shalit \\
  Technion\\
  \texttt{urishalit@technion.ac.il}
}

\begin{document}
\maketitle

\begin{abstract}
    Out-of-domain (OOD) generalization is a significant challenge for machine learning models. Many techniques have been proposed to overcome this challenge, often focused on learning models with certain invariance properties. In this work, we draw a link between OOD performance and model calibration, arguing that calibration across multiple domains can be viewed as a special case of an invariant representation leading to better OOD generalization. Specifically, we show that under certain conditions, models which achieve \emph{multi-domain calibration} are provably free of spurious correlations. This leads us to propose multi-domain calibration as a measurable and trainable surrogate for the OOD performance of a classifier. We therefore introduce methods that are easy to apply and allow practitioners to improve multi-domain calibration by training or modifying an existing model, leading to better performance on unseen domains. Using four datasets from the recently proposed WILDS OOD benchmark \cite{koh2020wilds}, as well as the Colored MNIST dataset \cite{kim2019learning}, we demonstrate that training or tuning models so they are calibrated across multiple domains leads to significantly improved performance on unseen test domains. We believe this intriguing connection between calibration and OOD generalization is promising from both a practical and theoretical point of view.
\end{abstract}

\section{Introduction}
\label{sec:intro}
Machine learning models have recently displayed impressive success in a plethora of fields \cite{huang2017densely, devlin2019bert, DBLP:journals/nature/Senior0JKSGQZNB20}. However, as models  are typically only trained and tested on in-domain (ID) data, they often fail to generalize to out-of-domain (OOD) data \cite{koh2020wilds}. The problem is especially pressing when deploying machine learning models in the wild, where they are required to perform well under conditions that were not observed during training. For instance, a medical diagnosis system trained on patient data from a few hospitals could fail when deployed in a new hospital.

Many methods have been proposed to improve the OOD generalization of machine learning models. Specifically, there is rapidly growing interest in learning models that display certain invariance properties under distribution shifts and do not rely on spurious correlations in the training data \cite{peters2016causal, heinze2018invariant, arjovsky2019invariant}. While highlighting the need for learning robust models, so far these attempts have limited success scaling to realistic high-dimensional data, and in learning truly invariant representations~\cite{rosenfeld2020risks,gulrajani2020search,kamath2021does}.

In this paper, we argue that an alternative and relatively simple approach for learning invariant representations could be achieved through model calibration across multiple domains. Calibration asserts that the probabilities of outcomes predicted by a model match their true probabilities. Our claim is that simultaneous calibration over several domains can be used as an observable indicator for favorable performance on unseen domains. For example, if we take all patients for whom a classifier outputs a probability of $0.9$ for being ill, and in one hospital the true probability of illness in these patients is $0.85$ while in the other it is $0.95$, then we may suspect the classifier relies on spurious correlations. Intuitively, the features which lead the classifier to predict a probability of $0.9$ imply different results under different experimental conditions, suggesting that their correlation with the label is potentially unstable. Conversely, if the true probabilities in both hospitals match the classifier's output, it may be a sign of its robustness. 

Our contributions are as follows:
We prove that in Gaussian-linear models, under a general-position condition, being concurrently calibrated across a sufficient number of domains guarantees a model has no spurious correlations.
We then introduce three methods for encouraging multi-domain calibration in practice. These are, in ascending order of complexity: (i) model selection by a multi-domain calibration score, (ii) robust isotonic regression as a post-processing tool, and (iii) directly optimizing deep nets with a multi-domain calibration objective, based on the method introduced by Kumar et al. \cite{kumar2018trainable}. We show that multi-domain calibration achieves the correct invariant classifier in a learning scenario presented by Kamath et al. \cite{kamath2021does}, unlike the objective proposed in Invariant Risk Minimization \cite{arjovsky2019invariant}. Finally, we demonstrate that the proposed approaches lead to significant performance gains on the WILDS benchmark datasets \cite{koh2020wilds}, and also succeed on the colored MNIST dataset \cite{kim2019learning}.

\section{Calibration and Invariant Classifiers}
\label{sec:cal_invar}

\subsection{Problem Setting}
Consider observable features $X$, a label $Y$ and an environment (or domain) $E$ with sample spaces $\cX, \cY, \cE$ accordingly. We mostly focus on regression and binary classification, therefore $\cY=\reals$ or $\cY=\{0, 1\}$. To lighten notation, our definitions will be given for the binary classification setting and we will point out adjustments to regression where necessary. There is no explicit limitation on $|\cE|$, but we assume that training data that has been collected from a finite subset of the possible environments $E_{\text{train}}\subset \cE$. The number of training environments is denoted by $k$, and $E_{\text{train}}=\{e_i\}_{i=1}^{k}\subset \cE$, so that our training data is sampled from a distribution $P[X, Y \mid E=e_i] \quad \forall i\in{[k]}$. Our goal is to learn models that will generalize to new, unseen environments in $\mathcal{E}$.

Ideally, we would like to learn a classifier that is optimal for all environments $\cE$. Unfortunately, we only observe data from the limited set $E_{\text{train}}$ and even if this set is extremely large, the Bayes optimal classifiers on each environment do not necessarily coincide. Following other recent work \cite{peters2016causal,heinze2018invariant,arjovsky2019invariant} we therefore aim for a different goal -- learning classifiers whose per-instance output will be stable across environments $E$, as we explain below. 

We assume the data generating process for $E,X,Y$ follows the causal graph in \figref{scm}. \footnote{See Appendix \ref{sec:scm} for a brief introduction to causal graphs.}
We differentiate between causal and anti-causal components of $X$, and further differentiate between the anti-causal variables which are affected or unaffected by $E$, denoted as $\Xac$ and $\Xacn$, respectively. As an illustrative example, consider again predicting illness across different hospitals. When predicting lung cancer, $Y$, from patient health records, $\Xc$ could be features like smoking. $\Xacn$ are symptoms of $Y$ such as infections that appear in chest X-rays, while $\Xac$ can be marks that technicians put on X-rays as in \cite{zech2018variable}. Smoking habits may vary across hospital populations, as might X-ray markings; but the influence of smoking on cancer and the manifestation of cancer in an X-ray do not vary by hospital.

        

We do not assume to know how to partition $X$ into $\Xc, \Xac, \Xacn$.
The main assumptions made in the causal graph in Fig. \ref{scm} are that there are no hidden variables, and that there is no edge directly from environment $E$ to the label $Y$. Such an arrow would imply the conditional distribution of $Y$ given $X$ can be arbitrarily different in an unseen environment $E$, compared to those present in the training set. Note that for simplicity we do not include arrows from $\Xc$ to $\Xac$ and $\Xacn$ but they may be included as well. 

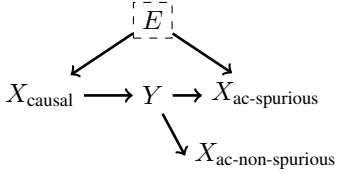
\begin{figure}[ht]
    \centering
    \begin{tikzpicture}
        \node[text centered] (xc) {$\Xc$};
        \node[right of = xc, node distance=1.5cm, text centered] (y) {$Y$};
        \node[right of = y, node distance=1.5cm, text centered] (xac) {$\Xac$};
        \node[below of = xac, node distance=0.8cm, text centered] (xns) {$\Xacn$};
        \node[draw, rectangle, dashed, above of = y, node distance=1cm, text centered] (e) {$E$};

        \draw[->, line width = 1] (xc) --  (y);
        \draw[->, line width = 1] (y) -- (xac);
        \draw[->, line width = 1] (e) -- (xc);
        \draw[->, line width = 1] (e) -- (xac);
        \draw[->, line width = 1] (y) to (xns.west);
    \end{tikzpicture}
    
    \caption{Learning in the presence of causal and anti-causal features. Anti-causal features can be either spurious ($\Xac$), or non-spurious ($\Xacn$).}
    \label{scm}
\end{figure}

We will say a representation $\Phi(X)$ contains a \emph{spurious correlation} with respect to the environments $E$ and label $Y$, if $Y \nindep E \mid \Phi(X)$; this motivates our naming of $\Xac$ and $\Xacn$ in Fig. \ref{scm}, as $Y \nindep E \mid \Xac$ but $Y \indep E \mid \Xacn$. Similar observations have been made by \cite{heinze2018invariant, arjovsky2019invariant}. Having a spurious correlation  implies that the relation between $\Phi(X)$ and $Y$ depends on the environment -- it is not transferable nor stable across environments. 
In this work we will simply consider the output $f(X)$ of a classifier $f: \X \rightarrow [0,1]$ as a representation. 
The crux of this paper is the observation that having $\E[Y \mid f(X), E=e]=f(X)$ for every value of $E$, i.e. $f$ being a \emph{\textbf{calibrated}} classifier across all environments, is equivalent up-to a simple transformation to having $Y \indep E \mid f(X)$, and thus to $f$ having \emph{\textbf{no spurious correlations}} with respect to $E$. We prove this assertion in section \ref{sec:ind_and_calib}, and as a demonstration of this principle we prove (section \ref{sec:motiv}) that linear models which are calibrated across a diverse set of environments $E$ are guaranteed to discard $\Xac$ as viable features for prediction.

\subsection{Invariance and Calibration on Multiple Domains}\label{sec:ind_and_calib}
We define calibration, along with a straightforward generalization to the multiple environment setting.
\begin{definition}
    Let $P[X, Y]$ be a joint distribution over the features and label, and $f:\X\rightarrow [0,1]$ a classifier. Then $f(\x)$ is calibrated w.r.t to $P$ if for all $\alpha\in{[0,1]}$ in the range of $f$, $\E_{P}{\left[ Y \mid f(X)=\alpha\right]} = \alpha$.
    In the multiple environments setting, $f(\x)$ is calibrated on $E_{\text{train}}$ if for all $e_i\in{E_{\text{train}}}$ and $\alpha$ in the range of $f$ restricted to $e_i$,
        $\E{\left[ Y \mid f(X)=\alpha, E=e_i\right]} = \alpha \nonumber$.
\end{definition}
For regression problems, we consider regressors that output estimates for the mean and variance of $Y$, and say they are calibrated if they match the true values similarly to the definition above. The precise definition can be found in the supplementary material.

We now tie the notion of calibration on multiple environments with OOD generalization, starting with its correspondence with our definition of spurious correlations. Recall that a representation $\Phi(X)$ does not contain spurious correlations if $Y \indep E \mid \Phi(X)$. Treating the output $f(X)$ of a classifier as a representation of the data, and considering classifiers satisfying the above conditional independence with respect to training environments, we arrive at a definition of an invariant classifier.
\begin{definition}
    Let $f:\X\rightarrow [0,1]$. $f$ is an \emph{invariant classifier} w.r.t $E_{\text{train}}$ if for all $\alpha\in{[0,1]}$ and environments $e_i,e_j\in{E_{\text{train}}}$, where $\alpha$ is in the range of $f$ restricted to each of them:
    \begin{align} \label{eq:invariant_predictor}
        \E[ Y \mid f(X)&=\alpha, E=e_i ] =  \E{\left[ Y \mid f(X)=\alpha, E=e_j \right]}.
    \end{align}
\end{definition}
Lemma \ref{lemma:corresp} gives the correspondence between invariant classifiers and classifiers calibrated on multiple environments. The proof is in \secref{sec:calibration_intro} of the supplementary material.

\begin{lemma}\label{lemma:corresp}
    If a binary classifier $f$ is invariant w.r.t $E_{\text{train}}$, then there exists some $g:\reals\rightarrow [0,1]$ such that (i) $g\circ f$ is calibrated on all training environments, and (ii) the mean squared error of $g\circ f$ on each environment does not exceed that of $f$. On the other hand, if a classifier is calibrated on all training environments it is also invariant w.r.t $E_{\text{train}}$.
\end{lemma}

Now, we can note how the above notion of invariance relates to that of Invariant Risk Minimization \cite{arjovsky2019invariant}, where invariance of a representation $\Phi:\cX\rightarrow\mathcal{H}$ is linked to a shared classifier $\w^*:\mathcal{H}\rightarrow [0, 1]$,  $\w^*\circ \Phi$ being optimal on all environments w.r.t a loss $l:[0,1]\times\mathcal{Y}\rightarrow \reals_{\geq 0}$. Under the representation $\Phi(X)=f(X)$, and the cross-entropy or squared losses it turns out that the original IRM definition coincides with \eqref{eq:invariant_predictor} \footnote{See Observation 2 in \cite{kamath2021does} for a proof.}. Hence we aim for a similar notion of conditional independence, yet we approach it from the point-of-view of calibration.
In \secref{sec:algs} we will see that taking this approach leads to different methods that are highly effective in achieving and assessing invariance.
We further note that the original IRM objective was deemed too difficult to optimize by the original IRM authors, leading them to propose an alternative called IRMv1. This alternative however does not capture the full set of required invariances, as shown by \cite{kamath2021does}, whereas we show in section \ref{subsec:2bit} that multi-domain calibration does indeed capture the required invariances.

Having established the connection between calibration on multiple environments and invariance, there are several interesting questions and points to consider: \\
\textbf{Calibration and sharpness.} Calibration alone is not enough to guarantee that a classifier performs well; on a single environment, always predicting $\E[Y]$ will give a perfectly calibrated classifier. Hence, multi-domain calibration should be combined with some sort of guarantee on accuracy. In the calibration literature, this is often referred to as sharpness. To this end, in \secref{sec:algs} we will propose regularizing models during training or fine-tuning with Calibration Loss Over Environments (\ours{}). Combining this regularizer with standard empirical loss functions helps balance between sharpness and multi-domain calibration. Even without training a new model, we will propose methods for model selection and post-processing that are very easy to apply and help improve multi-domain calibration without a significant effect on the sharpness of the models.
\\
\textbf{Generalization and dependence on $\Xac$.} Suppose that $f(X)$ is calibrated on $E_{\text{train}}$. Under what conditions does this imply it is calibrated on $\cE$? It is easy to show that calibration on several environments entails calibration on any distribution which can be expressed as a linear combination of the distributions underlying said environments. However, can we go beyond that? Given a general set $\cE$ we would like to know what conditions and how many training environments are required for calibration to generalize.
We also wish to understand when does calibration over a finite set of training environments indeed guarantee that a classifier is free of spurious correlations. We now turn to answer these questions in the setting of linear-Gaussian models.
\section{Motivation: a Linear-Gaussian Model}
\label{sec:motiv}
Let us consider data where $X$ is a multivariate Gaussian. Since we will be considering Gaussian data, the set of all environments $\cE$ will be parameterized using pairs of real vectors expressing expectations and positive definite matrices of an appropriate dimension expressing covariances: $\cE = \{(\mu, \Sigma) \mid \mu\in{\reals^d}, \Sigma\in{\mathbb{S}^d_{++}} \}$.

For two scenarios ((a) and (b) in Figure \ref{theoretical_cases}) we prove that when provided with data from $k$ training environments, where $k$ is linear in the number of features, and the environments satisfy some mild non-degeneracy conditions, any predictor that is calibrated on all training environments will not rely on any of the spurious features $X_{\text{ac-sp}}$, and will also be calibrated on all $e\in{\cE}$.

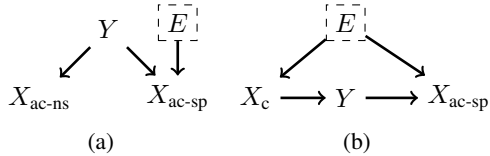
\begin{figure}[ht]
    \vspace{-0.3cm}
    \centering
    \subfigure[]{
    \begin{tikzpicture}
        \node[text centered] (y) {$Y$};
        \node[below right of = y, node distance = 1.3cm, text centered] (xac) {$X_{\text{ac-sp}}$};
        \node[above of = xac, node distance = 1.cm, draw, rectangle, dashed, text centered] (e) {$E$};
        \node[below left of = y, node distance = 1.3cm, text centered] (xns) {$X_{\text{ac-ns}}$};
   
        \draw[->, line width = 1] (y) -- (xac);
        \draw[->, line width = 1] (e) -- (xac);
        \draw[->, line width = 1] (y) to (xns);
    \end{tikzpicture}}
    \subfigure[]{
    \begin{tikzpicture}
        \node[draw, rectangle, dashed, text centered] (e) {$E$};
        \node[below of = e, node distance=1.cm, text centered] (y) {$Y$};
        \node[right of = y, node distance=1.5cm, text centered] (xac) {$X_{\text{ac-sp}}$};
        \node[left of = y, node distance=1.2cm, text centered] (xc) {$X_{\text{c}}$};
        \draw[->, line width = 1] (e) -- (xc);
        \draw[->, line width = 1] (xc) -- (y);
        \draw[->, line width = 1] (y) -- (xac);
        \draw[->, line width = 1] (e) -- (xac);
    \end{tikzpicture}}
    \vspace{-0.2cm}
    \caption{Graphs describing the two cases in our theoretical analysis. We use acronyms in subscripts to lighten notation. (a) All features are anti-causal, some are spurious while others are invariant. (b) Features are either causal and may undergo covariate shift, or are anti-causal and spurious.}
    \label{theoretical_cases}
\end{figure}

   

In scenario (a), we take $Y$ to be a binary variable drawn from a Bernoulli distribution with parameter $\eta\in{[0,1]}$, and observed features are generated conditionally on $Y$. The features $\x_{\text{ac-ns}}\in{\reals^{d_{\text{ns}}}}$ are invariant, meaning their conditional distribution given $Y$ is the same for all environments, whereas $\x_{\text{ac-sp}}\in{\reals^{d_{\text{sp}}}}$  are spurious features, as their distribution may shift between environments, altering their correlation with $Y$. The data generating process for training environment $i\in{[k]}$ in Fig. \ref{theoretical_cases}(a) is given by:

\begin{minipage}{.32\linewidth}
    \begin{align*}
    y = \begin{cases}
    1 & \text{w.p }  \eta \\
    0 & \text{o.w}
    \end{cases}
    \end{align*}
\end{minipage}
\begin{minipage}{.65\linewidth}
    \begin{align} \label{eq:setting_a}
        X_{\text{ac-ns}} &\mid Y=y \sim\N\left((y-1/2)\mu_{\text{ns}}, \Sigma_{\text{ns}}\right), \nonumber \\
        X_{\text{ac-sp}} &\mid Y=y \sim\N\left((y-1/2)\mu_i, \Sigma_i\right).
    \end{align}
\end{minipage}

For $\x= [\x_{\text{ac-ns}}, \x_{\text{ac-sp}}]$ we consider a linear classifier $f(\x ; \w,b) = \sigma( \w^\top \x  + b)$, where $\sigma: \reals \rightarrow [0,1]$ is some invertible function (e.g. a sigmoid). Since the mean of spurious features, $\mu_i$, is determined by $y$, these features can help predict the label in some environments. Yet, these correlations do not carry to all environments, and $f(\x)$ might rely on spurious correlations whenever the coefficients in $\w$ corresponding to $\x_{\text{ac-sp}}$ are non-zero.
Any such classifier can suffer an arbitrarily high loss in an unseen environment, because a new environment can reverse and magnify the correlations observed in $E_{\text{train}}$.
Using these definitions, we may now state our result for this case:
\begin{theorem} \label{thm:setting_a}
Given $k > 2d_{\text{sp}}$
training environments where data is generated according to \eqref{eq:setting_a} with parameters $\{\mu_i, \Sigma_i\}_{i=1}^{k}$, we say they lie in general position if for all non-zero $\x\in{\reals^{d_{\text{sp}}}}$:
\begin{align*}
    \mathrm{dim}\left(\mathrm{span}\left\{\begin{bmatrix} \Sigma_i\x + \mu_i \\
    1 \end{bmatrix}\right\}_{i\in{[k]}}\right) = d_{\text{sp}}+1.
\end{align*}
If a linear classifier is calibrated on $k$ training environments which lie in general position, then its coefficients for the features $\x_{\text{ac-sp}}$ are zero. Moreover, the set of training environments that do not lie in general position has measure zero in the set of all possible training environments $\cE^k$.
\end{theorem}
As a corollary, we see that calibration on training environments generalizes to calibration on $\cE$. The proof of this theorem is given in the supplementary material, \secref{sec:proof1}. The data generating process closely resembles the one considered by \cite{rosenfeld2020risks}, who use diagonal covariance matrices.

In the second scenario we consider the addition of causal features subject to covariate shift $\x_{\text{c}}\in{\reals^{d_{\text{c}}}}$, as shown in \figref{theoretical_cases}b. The covariate shift is induced when the environments $E$ alter the distribution of the causal features $\x_{\text{c}}$ \cite{scholkopf2012causal}. In this case, we analyze a regression problem since it is amenable to exact analysis. The data generating process for training environment $i\in{[k]}$ is:
\begin{align} \label{eq:setting_b}
    X_{c} \sim  \N(\mu^c_i, &\Sigma^c_{i}); \:
    Y = {\w^*_c}^\top \x_c + \xi, \: \xi\sim\N(0, \sigma^2_y) \nonumber \\
    &X_{\text{ac-sp}} = y\mu_i + \eta, \: \eta\sim \N(\textbf{0},\Sigma_i).
\end{align}
For $\x= [\x_c, \x_{\text{ac-sp}}]$ it turns out that in this case, calibration on multiple domains forces $f(\x)$ to discard $\x_{\text{ac-sp}}$, but also forces it to use $\w_c^*$, since it characterizes $P(Y \mid \x_c)$ which is the invariant mechanism in this scenario. The exact statement and proof are in \secref{sec:proof2} of the supplement.
\begin{theorem}[informal]
Let $f(\x; \w) = \w^\top\x$ be a linear regressor and assume we have $k > \max{\{d_\text{c}+2, d_{\text{sp}}\}}$ training environments where data is generated according to \eqref{eq:setting_b}. Under mild non-degeneracy conditions,
if the regressor is calibrated across all training environments then the coefficients corresponding to $X_{\text{c}}$ equal $\w_c^*$ and those that correspond to $X_{\text{ac-sp}}$ are zero.
\end{theorem}

Together, these results show calibration can generalize across environments, given that the number of environments is approximately that of the spurious features. They also show that for the settings above, the relatively stable and well-known notion of calibration implies avoiding spurious correlations. 

\section{Related Work}
As discussed in Section \ref{sec:cal_invar}, multi-domain calibration is an instance of an invariant representation \cite{arjovsky2019invariant}. Many extensions to the above work have been proposed, e.g. \cite{krueger2020out, bellot2020generalization}. Yet, recent work claims that many of these approaches still fail to find invariant relations in cases of interest \cite{kamath2021does, rosenfeld2020risks, guo2021out}, where a significant challenge seems to be the gap between what is achieved by the regularization term used in practice and the goal of conditional independence $Y \indep E \mid \Phi(X)$. Gulrajani et al. \cite{gulrajani2020search} give a sobering view on methods for OOD generalization, emphasizing the power of ERM and data augmentation, and the challenge of model selection. We claim that compared to the above approaches, multi-domain calibration studied here is a simpler form of invariance. Furthermore, calibration is attractive because there are standard tools to quantify it such as calibration scores \cite{DBLP:conf/cvpr/NixonDZJT19} and a vast literature on its properties and how it can be obtained \cite{DBLP:conf/icml/ZadroznyE01,vovk2003self,niculescu2005predicting,kumar2018trainable,vaicenavicius2019evaluating,gupta2020distribution,rahimi2020intra}.

Learning models which generalize OOD is a fruitful area of research with many recent developments. Most work focuses on the case of Domain Adaptation where unlabeled samples are available from the target domain, including recent work on OOD calibration \cite{wang2020transferable}. However, important work has also been done on the area of our focus -- the so-called ``proactive'' case \cite{subbaswamy2019preventing}, where no OOD samples are available whatsoever \cite{magliacane2018domain,heinze2018invariant, rothenhausler2018anchor,peters2016causal, sagawa2019distributionally}. 

Calibration also plays an important role in uncertainty estimation for deep networks \cite{guo2017calibration}, and recently in fairness, where calibration on subgroups of populations is sought \cite{pleiss2017fairness}. This has interesting resemblance to the multiple environments calibration we consider here. A more general notion of multi-calibration has also been studied in this context \cite{hebert2018multicalibration}, with recent results on sample complexity \cite{shabat2020sample} which may provide tools to finite sample analysis of domain generalization. Finally, multiple methods for training calibrated models \cite{kumar2018trainable, mukhoti2020calibrating, rahimi2020intra} have also been proposed. In Section \ref{sec:algs} we propose a generalization of \cite{kumar2018trainable} to the multi-domain case to achieve multi-domain calibration.
\section{Proactively Achieving Multi-Domain Calibration} \label{sec:algs}
So far we have seen a general argument why calibration can limit spurious correlations, and that in linear-Gaussian models multi-domain calibration guarantees OOD generalization. Now we turn to a more applied perspective and show how can we optimize models so they achieve this type of calibration in practice. We propose three approaches: (1) using calibration measures for model selection, (2) post-processing calibration, and (3) a calibration objective building on a method proposed by \cite{kumar2018trainable}. \secref{sec:calibration_intro} in the supplementary provides a slightly broader introduction to notions we use here. 
We will assess model calibration by the Expected Calibration Error (\emph{ECE}) of the calibration curve~\cite{degroot1983comparison}, which is the average deviation between model accuracy and model confidence.


\subsection{Model selection with average ECE} \label{sec:model_selection} Model selection is challenging when aimed at OOD generalization. As recently observed by \cite{gulrajani2020search}, since OOD accuracy is often at odds with In-Domain (ID) accuracy, selection based on ID validation error eliminates the advantage of domain generalization methods over vanilla ERM with data augmentation. We suggest that model selection towards OOD generalization should balance ID validation error with another observable surrogate for the stability of a model to distribution shifts between domains. Motivated by multi-domain calibration, we propose using the average ECE across training environments as this surrogate. Concretely, we propose choosing a model with lowest average ECE from those obtaining ID validation accuracy that is above a certain user-defined threshold.
\subsection{Post-Processing Calibration}\label{subsec:postproc}
Practitioners interested in (single-domain) calibrated models often apply post-processing calibration methods to binary classifiers, where the most widely used approach is Isotonic Regression Scaling \cite{DBLP:conf/icml/ZadroznyE01, niculescu2005predicting}. 
Unlike standard calibration problems, in our case there are multiple domains to calibrate over. We give two ways of extending Isotonic Regression to the multi-domain setting, which we term ``naive calibration'' and ``robust calibration''.
\textbf{Naive Calibration} takes predictions of a trained model $f$ on validation data pooled from all domains and fits an isotonic regression $z^*$. We then report the performance of $z^* \circ f$ on the OOD test set.\\
\textbf{Robust Calibration:} In a multiple domain setting, Naive calibration may produce a model that is well calibrated on the pooled data, but uncalibrated on individual environments.
Since our goal is simultaneous calibration, the following alternative attempts to bound the worst-case miscalibration across training environments. For each environment $e\in{E_{\text{train}}}$, we denote the number of validation examples we have from it by $N_{e}$, and by $f_{e,i}$ the prediction of the model on the $i$-th example. Then in a similar vein to robust optimization, we fit an isotonic regressor that solves:
$\begin{aligned} \label{eq:robust_iso}
    z^* =  \argmin_{z}\max_{e\in{E_{\text{train}}}}{\frac{1}{N_e}\sum_{i=1}^{N_e}{\left(z(f_{e,i})-y_i\right)^2}}.
\end{aligned}$
Since Isotonic Regression can be formulated as a quadratic program, and \eqref{eq:robust_iso} minimizes a pointwise maximum over such objectives, we can cast Eq. \ref{eq:robust_iso} as a convex program and solve with standard optimizers. 
We then evaluate the OOD performance of $z^* \circ f$.

\subsection{Learning with Multi-Domain Calibration Error}\label{subsec:clove}
The above model selection and post-processing methods are easy to apply and (as we will soon see) surprisingly effective. However, both are limited in their power to learn a model that is truly well-calibrated across multiple domains. We now propose a more powerful approach: an objective function that directly penalizes calibration errors on multiple domains during training.
Specifically, we propose learning a parameterized classifier $f_\theta(\x)$ using a learning rule of the form:
    $\min_{\theta}{\sum_{e\in{E_{\text{train}}}}{l^e(f_\theta)} + \lambda \cdot r(f_\theta) }$,
where $l:\reals \times \reals \rightarrow \reals$ is an empirical loss function (e.g. cross-entropy) and $l^e(f_\theta)$ denotes the expected loss over data from training environment $e$, and $r(f_\theta)$ is a regularization term over multiple environments. Using this notation the method proposed by \cite{arjovsky2019invariant} learns a classifier $f=w\circ \Phi$ with a regularizer given by $r(f) = \sum_{e\in{E_{\text{train}}}}{r^e_{\text{IRMv1}}}(f)$, where $r^e_{\text{IRMv1}}(f) = \|\nabla_{w \mid w=1}{l^e(w\cdot \Phi)}\|^2$. 

Our proposed regularizer $r(f_\theta)$ is based on the work of Kumar et al. \cite{kumar2018trainable}, who introduce a method they call Maximum Mean Calibration Error (MMCE). MMCE harnesses the power of universal kernels to express the ECE as an Integral Probability Measure, and works as follows: For a dataset $D = \{\x_i, y_i\}_{i=1}^{m}$, denote the confidence of a classifier on the $i$-th example by $f_{\theta;i}=\max\{f_\theta(x_i), 1-f_\theta(x_i)\}$
and its correctness by $c_i=\mathbbm{1}_{| y_i-f_{\theta;i} | < \frac{1}{2}}$.
For a given universal kernel $k:\reals \times \reals \rightarrow \reals$, MMCE over the dataset $D$ is given by:
    $r^{D}_{\text{MMCE}}(f_\theta) = \frac{1}{m^2}\sum_{i,j\in{D}}{(c_i-f_{\theta;i})(c_j-f_{\theta;j})k(f_{\theta;i},f_{\theta;j})}$.
\textbf{Calibration Loss Over Environments (\ours{}).} 
Given multiple training domains with a dataset $D^e$ for each $e\in{E_{\text{train}}}$, we arrive at our proposed regularizer by aggregating MMCE over them:
$r_{\text{\ours{}}}(f_\theta) = \sum_{e\in{E_{\text{train}}}}r^{D_e}_{\text{MMCE}}(f_\theta)$.
A key property of \ours{} is that its minima correspond to perfectly calibrated classifiers over all training domains, a consequence of the correspondence between MMCE and perfect calibration.
\begin{corollary}[of Thm.~1 in \cite{kumar2018trainable}]\label{corr:proper_score}
\ours{} is a proper scoring rule. That is, it equals $0$ if and only if $f_\theta(\x)$ is perfectly calibrated for every $e \in E_{\text{train}}$.
\end{corollary}
Additional properties of \ours{}, such as large deviation bounds and relation to ECE, can also be derived; see results in \cite{kumar2018trainable} for further details. In the following section, we will see how these properties translate into favorable OOD generalization in practice when training with \ours{}.

\section{Experiments and Results}
\label{sec:exp}

\subsection{Colored MNIST and Two-Bit Environments}\label{subsec:2bit}

In order to explore the challenges of OOD generalization and how they relate to learning from multiple environments, \cite{arjovsky2019invariant} used the colored MNIST dataset \cite{kim2019learning}. In this dataset certain digits tend to be colored either red or green in the train set, but the correlation between colors and digits is flipped in the OOD test set, making color a spurious feature. This dataset was then further simplified into ``Two-Bit'' environments by \cite{kamath2021does}, who proved that the IRMv1 penalty proposed in \cite{arjovsky2019invariant} does \emph{not} in fact achieve the correct invariant solution on the simplified setting.
The Two-Bit environments problem setting has two binary features, $X_1, X_2\in{\{-1, 1\}}$, corresponding respectively to digit identity ($0-4$ or $5-9$) and digit color in the original colored MNIST. 
The environments are parameterized by $e=(\alpha, \beta)\in{[0,1]^2}$ controlling the correlation of the features with the label:\\ 
   $ Y \leftarrow \mathrm{Rad}(0.5),~
    X_1 \leftarrow Y\cdot \mathrm{Rad}(\alpha), X_2 \leftarrow Y\cdot \mathrm{Rad}(\beta)$,
where $\mathrm{Rad}(\delta)$ is a random variable equal to $-1$ with probability $\delta$ and $1$ with probability $1-\delta$. At training we are given data from two environments $e_1=(\alpha, \beta_1), e_2=(\alpha, \beta_2)$, $\beta_1 \neq \beta_2$. The learned model is tested on a new environment $e_3=(\alpha, \beta_3)$ with $\beta_3$ significantly different from $\beta_1, \beta_2$. Only a model discarding the spurious feature $X_2$ will maintain its accuracy moving from train to OOD test.\\ \textbf{Calibration discards spurious correlation in Two-Bit environments.}
\figref{fig:mnist}(a), which we adapt from Figure 6 in Appendix B of \cite{kamath2021does}, illustrates the merits of \ours{} in this setting. The figure shows the space of odd classifiers, i.e. those for which $f(1,-1) = -f(-1,1)$, and $f(1,1) = -f(-1,-1)$.\footnote{As explained in \cite{kamath2021does}, the optimal solutions are odd so we may focus on them for visualization purposes.} The true invariant classifiers are those for which in addition $f(1,1) = f(1,-1)$, corresponding to models lying on the diagonal of \figref{fig:mnist}(a), denoted by the dashed gray line.
In the figure, we plot in solid lines the classifiers for which $r^e_{\text{IRMv1}}(f)$ equals $0$, and in solid circles the classifiers for which $r^e_{\text{MMCE}}(f)$ equals 0 (due to \corref{corr:proper_score} these coincide with calibrated classifiers on environment $e$). Note that in this parameterization, the zeros of $r^e_{\text{IRMv1}}(f)$ are lines whereas the zeros of $r^e_{\text{MMCE}}(f)$ are isolated points. 
Intersections of the zeros of $r^e_{\text{IRMv1}}(f)$ denote solutions for which the corresponding regularization terms are $0$ on all respective environments, while intersection of zeros of $r^e_{\text{MMCE}}(f)$ are the zeros of $r_{\text{\ours}}(f)$. As observed by \cite{kamath2021does}, when $E_{\text{train}}=\{e_1, e_2\}$ the solution denoted by $\text{OPT}_{\text{IRMv1}}$ has the lowest empirical loss, yet this solution has a spurious correlation with $X_2$ and thus will incur a higher loss on the test environment $e_3$. This means the corresponding IRMv1 learning rule cannot retrieve the optimal invariant classifier. On the other hand, learning with \ours{} does retrieve the optimal invariant classifier in this case, in addition to the trivial, constant classifier. This means \ours{} discards spurious correlations in cases where IRMv1 does not. In \secref{sec:cmnist_results} we present experiments reproducing the above scenario on the Colored MNIST dataset.

\textbf{Model selection based on average ECE} 
We train models with varying hyperparameters on Colored MNIST using ERM, \ours{} and IRM, (100 models with each algorithm, see \secref{sec:cmnist_results} of the supplement for details). We then calculate the ECE and IRMv1 penalties of each model over a held-out validation set from each training environment, and evaluate the average of these against OOD accuracy. \figref{fig:mnist}(b) presents the results across all trained models. The ID ECE penalty displays a very strong correlation across the entire range and every training regime (Pearson corr. = -0.92), while ID IRMv1 behaves more erratically (Pearson corr. = -0.59).
Since quantities used for model selection should be agnostic to choices made at training time, we suggest that ID ECE is a better choice for use in model selection. Further results on model selection can be found in the supplement, \secref{sec:cmnist_results}.

\begin{figure}[ht]
    \centering
    \subfigure[]{\centering
    \includegraphics[width=0.29\textwidth]{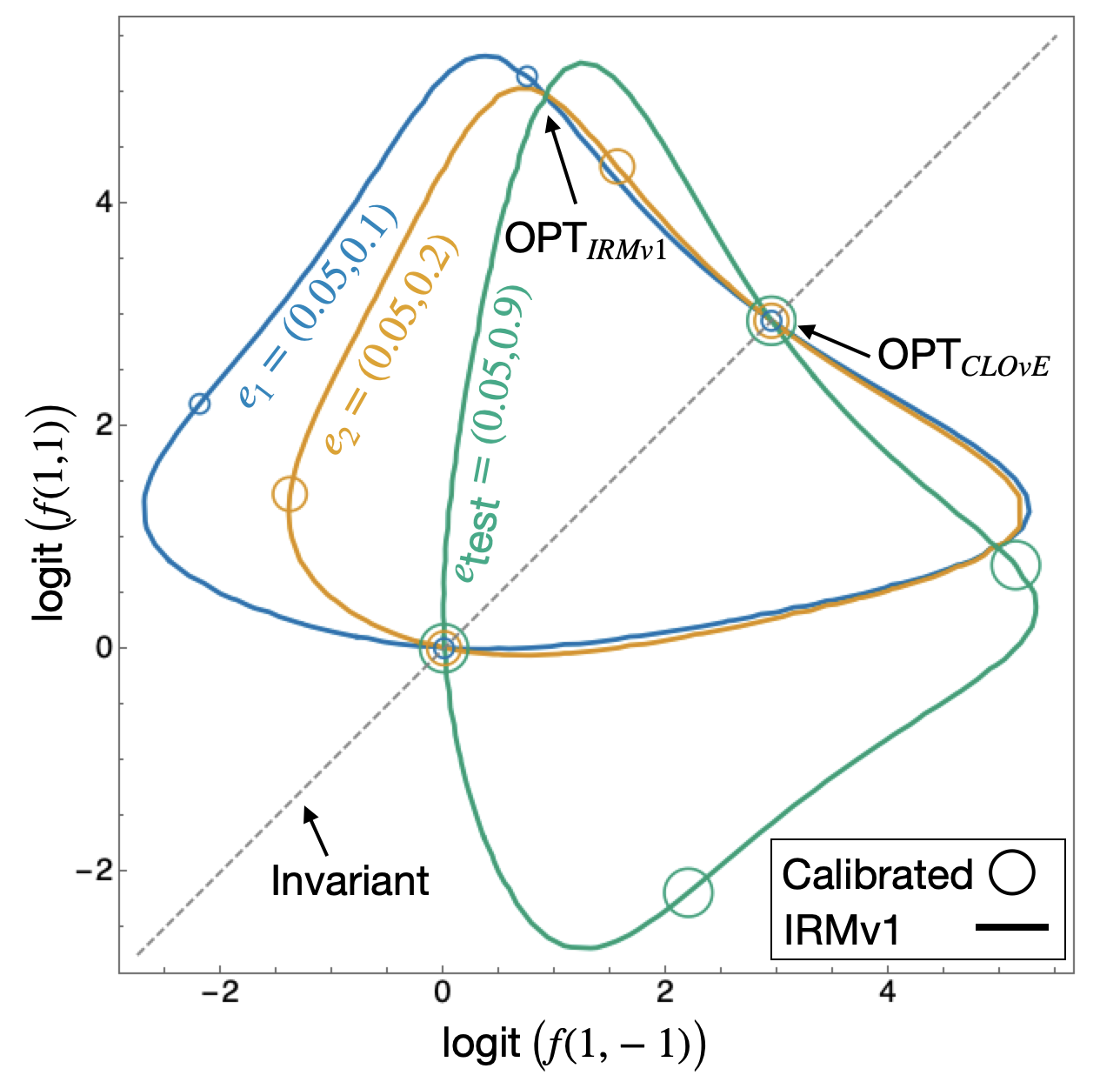}}
    \subfigure[]{\centering
    \includegraphics[width=0.69\textwidth]{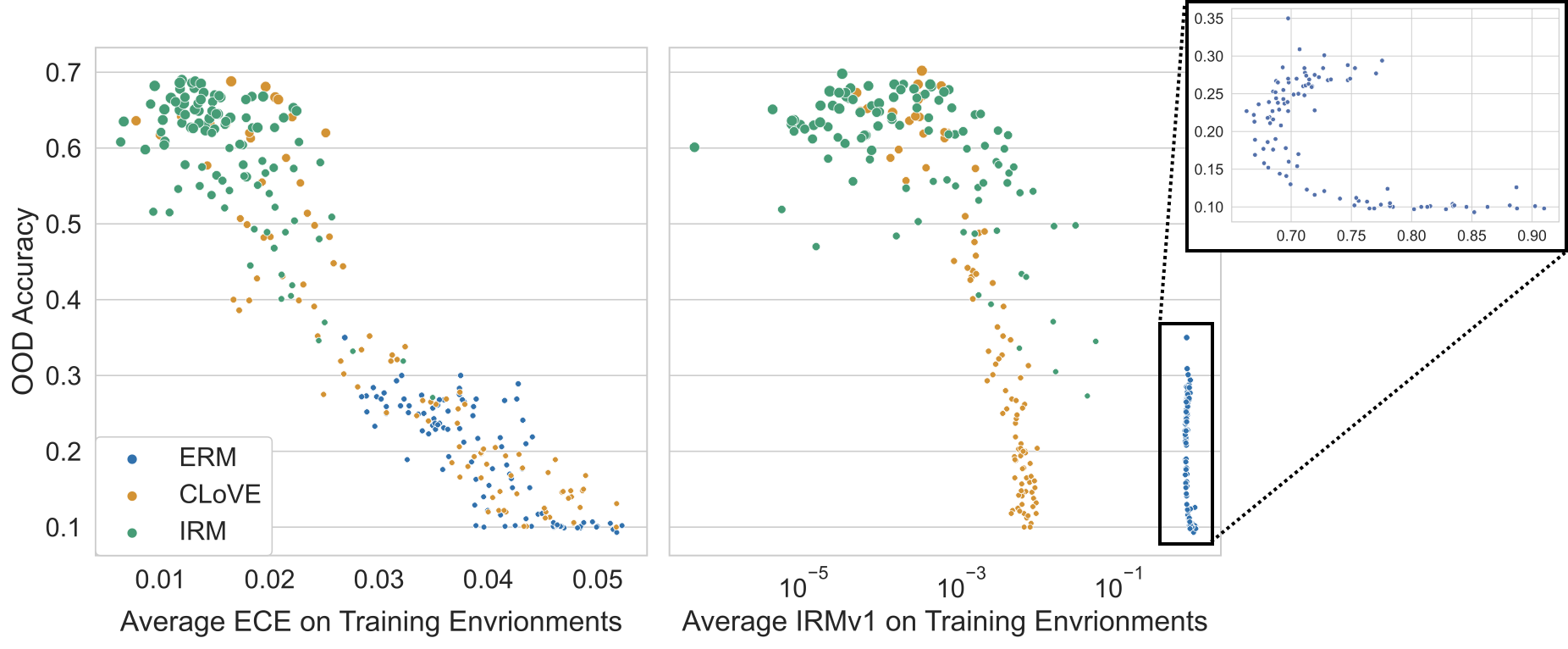}}
    \vspace{-0.2cm}
    \caption{(a) Zeros of MMCE and IRMv1 are indicated by circles and by solid lines respectively, in a color corresponding to each environment. The dashed diagonal is the space of invariant solutions. Some zeros intersect across environments, and these are therefore the domain-invariant solutions. Among the domain-invariant solutions, $\text{OPT}_{\text{IRMv1}}$ has the lowest empirical loss when training on $e_1, e_2$. Hence learning with IRMv1 will prefer this model over $\text{OPT}_{\text{\ours}{}}$, which discards the spurious correlation with $X_2$. (b) Correspondence between observable criteria and OOD accuracy in CMNIST. Each point corresponds to a model trained with some training algorithm (marked by color) and hyperparameter setting. Size of marker is proportional to the ratio between OOD and ID accuracies.}
    \label{fig:mnist}
\end{figure}

\subsection{WILDS Benchmarks}
\label{subsec:wilds}

WILDS is a recently proposed benchmark of in-the-wild distribution shifts from several data modalities and applications\footnote{\url{https://wilds.stanford.edu}}. Table \ref{tab:wilds} presents the four WILDS datasets we experiment with, chosen to represent diverse OOD generalization scenarios. We follow the models and training algorithms proposed by \cite{koh2020wilds}. In order to perform multi-domain calibration we modify the splits to include a multi-domain validation set whenever possible. See supplemental \secref{sec:datasets_models} for details and for additional results on Amazon Reviews. As in \cite{koh2020wilds}, we use three different training algorithms to train our models: \textbf{ERM}, \textbf{IRM}, \textbf{DeepCORAL}, and further use \textbf{GroupDRO} for one of the datasets, compatible with WILDS version 1.0.0. We apply three calibration approaches described in \ref{subsec:postproc} and \ref{subsec:clove} above to each trained model: \textbf{naive calibration} and \textbf{robust calibration}, which are post-processing methods and therefore applied on the models' outputs; and \textbf{\ours}, which we apply as a fine-tuning approach to the top layers of each trained model.
We train each (algorithm $\times$ calibration) combination four times with different random seeds, and report average results and their standard deviations.

\begin{table*}[htp!]
    \centering
    \scalebox{0.8}{
    \begin{tabular}{l|lllll} \toprule
        Dataset & Type & Label ($y$) & Input ($x$) & Domain ($e$) & Model ($f(x)$) \\ \midrule \midrule
        \textbf{PovertyMap} & Regression & Asset Wealth Index & Satellite Image & Country & ResNet \\
        \textbf{Camelyon17} & Binary & Tumor Tissue & Histopathological Image & Hospital & DenseNet \\
        \textbf{CivilComments} & Binary & Comment Toxicity & Online Comment & Demographics & BERT\\
        \textbf{FMoW} & Multi-class & Land Use Type & Satellite Image & Region & DenseNet \\ \bottomrule
    \end{tabular}}
    \caption{Description of each of the datasets used in our WILDS experiments.}
    \label{tab:wilds}
    \vspace{-0.3cm}
\end{table*}





Table \ref{tab:fmow_camelyon} presents our main results on the \textit{FMoW} (left) and \textit{Camelyon17} (right) datasets. On both datasets, robust calibration already improves performance, and CLOvE then significantly outperforms robust calibration, improving performance by $7\%$ and  $2.8\%$ (absolute) over the strongest alternative on \textit{FMoW} and \textit{Camelyon17}, respectively. When compared to the original model, the performance of CLOvE is even more striking, with CLOvE outperforming it by more than $10\%$ (absolute) on \textit{FMoW} and  $6\%$ on \textit{Camelyon17}. Another appealing property of CLOvE is the low variance exhibited across different runs. Indeed, CLOvE has lower variance than both naive and robust calibration approaches, and has lower variance than the original (uncalibrated) model on 4 of the 6 experiments. 

\begin{table}[ht]
    \centering
    \scalebox{0.8}{
    \begin{tabular}{l|cccc|cccc} \toprule
     & \multicolumn{4}{c}{\textit{FMoW}} &  \multicolumn{4}{c}{\textit{Camelyon17}} \\
    Algorithm &  Orig. & Naive Cal. & Rob. Cal. & CLOvE  & Orig. & Naive Cal. & Rob. Cal. & CLOvE \\ \midrule \midrule
    ERM & 32.63 & 33.09 & 37.19 & \textbf{44.16} & 66.66 & 71.23 & 71.22 & \textbf{75.75} \\
     & (\underline{1.6}) & (2.1) & (3.5) & (1.8) & (14.4) & (8.9) & (8.6) & (\underline{4.9}) \\
    DeepCORAL & 31.73 & 31.75 & 33.86 & \textbf{40.05} & 72.44 & 75.97 & 76.8 & \textbf{79.96} \\
     & (1.) & (1.) & (1.6) & (\underline{0.9}) & (4.4) & (5.4) & (6.5) & (\underline{3.9}) \\
    IRM & 31.33 & 31.81 & 34.41 & \textbf{42.24} & 70.87 & 73.25 & 73.4 & \textbf{73.95} \\ 
     & (\underline{1.2}) & (1.6) & (1.5) & (1.4) & (6.8) & (6.6) & (6.9) & (\underline{6.1}) \\ \bottomrule
    \end{tabular}}
    \caption{Left: worst unseen region accuracy on OOD test set in \textit{FMoW}. Right: Accuracy on unseen hospital test set in \textit{Camelyon17}. Orig.: original algorithm, no changes applied. Best OOD result for each domain in \textbf{bold}. Standard deviation across runs in brackets, lowest OOD std. is \underline{underlined}.}
    \label{tab:fmow_camelyon}
    \vspace{-0.5cm}
\end{table}

\textbf{Analysis.} As can be seen in Figure \ref{fig:camelyon_ece}, improvements in ID calibration are associated with better OOD performance. Interestingly, when our post-processing does not improve OOD performance, it is often linked to our inability to substantially improve ID calibration. This is most visible in IRM experiments, where robust calibration is unable to outperform naive calibration both in terms ID calibration and in OOD performance. Finally, we find it interesting that merely post-processing the data (as in robust calibration) can already have such a marked effect on OOD accuracy, though still inferior to actually optimizing for multi-domain calibration as done by \ours.  

\begin{wrapfigure}{r}{0.4\textwidth}
    \centering
    \vspace{-0.4cm}
    \includegraphics[width=0.4\textwidth]{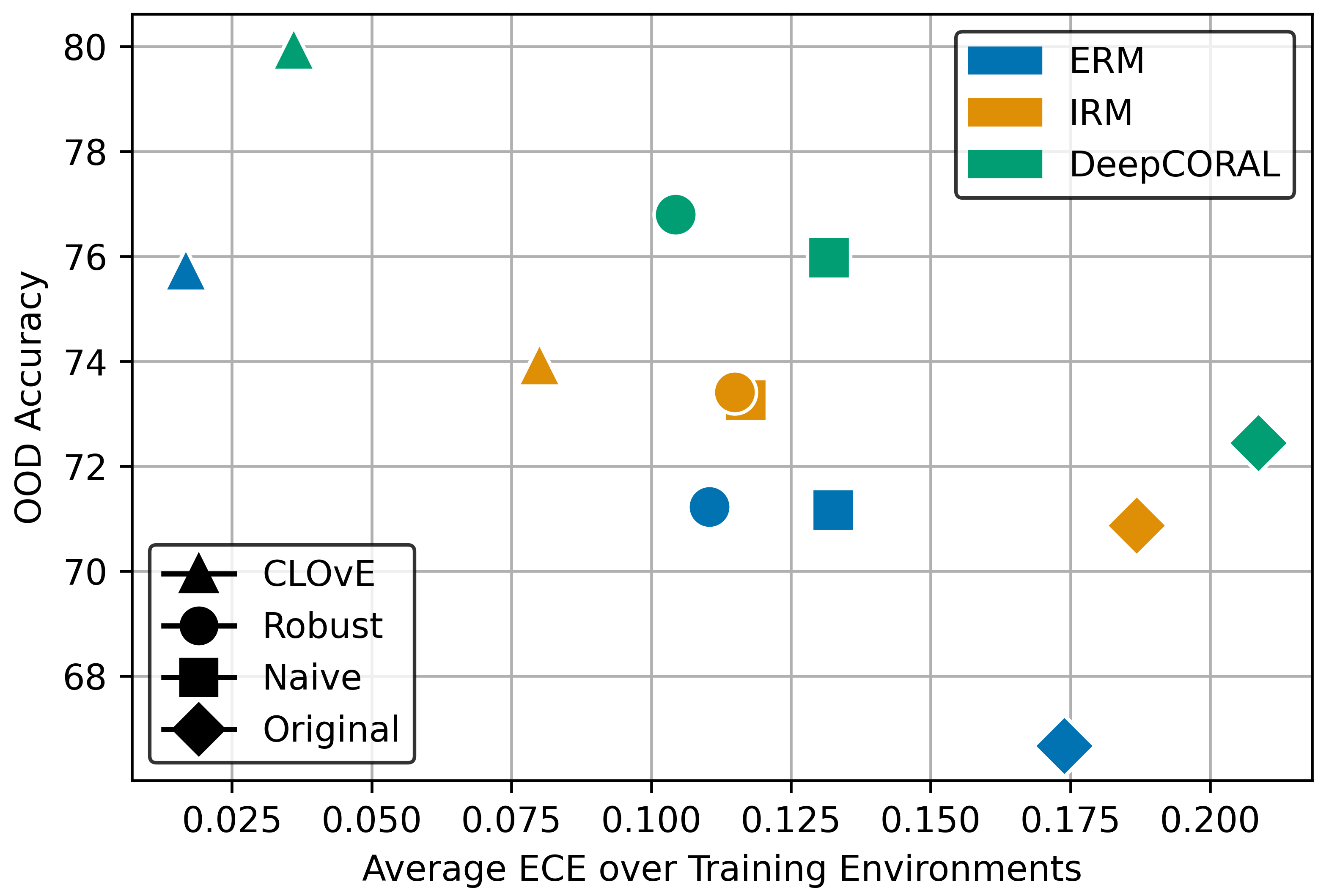}
    \vspace{-0.7cm}
    \caption{OOD accuracy as a function of average ECE over training domains, for all models on the \textit{Camelyon17} dataset.}
    \label{fig:camelyon_ece}
\end{wrapfigure}

\textbf{Results on alternative settings.} While our theoretical analysis is focused on OOD generalization of classification models, we also experiment with alternative settings from \textit{WILDS} to test the power of ID calibration in improving OOD performance. Specifically, we experiment with the \textit{PovertyMap} dataset, which introduces a regression task, and the \textit{CivilComments} dataset, which introduces a sub-population shift scenario for a binary classifier. As can be seen in Table \ref{tab:poverty_civilcomments}, results on the \textit{CivilComments} dataset (right), show that calibration consistently improves worst-case performance, with an average improvement of $21.5\%$ across training algorithms. While CLOvE does outperform naive and robust calibration on average, the gain is lower in comparison to \textit{FMoW} and \textit{Camelyon17}.

In \textit{PovertyMap} (left), the model solves a regression task, so we cannot use CLOvE to improve OOD performance. Still, robust calibration improves performance across all experiments, though by a smaller margin. In the case of models pre-trained by IRM, robust calibration improves OOD performance substantially, outperforming the original model by $0.08\%$ (absolute). Interestingly, calibration also leads to more stable results both in \textit{PovertyMap} and in \textit{CivilComments}, as can be seen in the standard deviation across different model runs.

\begin{table}[ht]
    \centering
    \scalebox{0.78}{
    \begin{tabular}{l|ccc|l|cccc} \toprule
      \multicolumn{4}{c|}{\textit{PovertyMap}} & & \multicolumn{4}{c}{\textit{CivilComments}} \\
     Algorithm  & Orig. & Naive Cal. & Rob. Cal. & Algorithm & Orig. & Naive Cal. & Rob. Cal. & CLOvE \\ \midrule \midrule
    ERM  & 0.832 & 0.827 & \textbf{0.834} & ERM & 63.65 & 76.98 & 78.99 & \textbf{80.39}  \\
      & (0.011) & (0.014) & (\underline{0.006}) &  & (2.6) & (\underline{0.5}) & (0.8) & (0.7) \\
    IRM  & 0.735 & 0.812 & \textbf{0.815} & IRM & 40.61 & \textbf{68.97} & 68.92 & 68.45 \\ 
      & (0.117) & (0.016) & (\underline{0.015}) & & (16) & (\underline{1.3}) & (\underline{1.3}) & (2.) \\
     DeepCORAL  & 0.832 & 0.835 & \textbf{0.837} & GroupDRO & 71.67 & 76.2 & 78.54 & \textbf{80.07} \\
      & (0.011) & (\underline{0.009}) & (0.012) &  & (0.7) & (1.3) & (0.8) & (\underline{0.3}) \\
     \bottomrule
    \end{tabular}}
    
    \caption{Left: Pearson correlation $r$ on in-domain (ID) and OOD (unseen countries) test sets in \textit{PovertyMap}. Right: average group accuracy on the test set in the \textit{CivilComments} dataset.}
    \label{tab:poverty_civilcomments}
    \vspace{-0.7cm}
\end{table}

\section{Conclusion}
\label{sec:conc}
In this paper we highlight a novel connection between multi-domain calibration and OOD generalization, arguing that such calibration can be viewed as an invariant representation. We proved in a linear setting that models calibrated on multiple domains are free of spurious correlations and therefore generalize out of domain. We then proposed multi-domain calibration as a practical and measurable surrogate for the OOD performance of a classifier. We demonstrated that actively tuning models to achieve multi-domain calibration significantly improves model performance on unseen test domains, and that in-domain calibration on a validation set is a useful criterion for model selection. 
A major limitation of our work is that our theoretical findings are limited to linear models in a population (as opposed to finite-sample) setting; we thus consider them more as a motivation rather than a full justification of using multi-domain calibration in practice as we do. 
Better formal understanding can also inform us on when should we expect to gain from calibration techniques. Even though in our experiments we see that the techniques mostly improve OOD performance while preserving ID accuracy, it is plausible that failure cases exist and should be characterized.
We look forward to expanding the scope of theoretical understanding of the conditions under which multi-domain calibration can provably guarantee out-of-domain generalization, including the finite-sample setting and the analysis of specific algorithms. We also expect new practical methods, building on our findings, will help push forward the real-world ability to generalize to unseen test domains.   

\section*{Acknowledgments}
We wish to thank Ira Shavitt for his helpful comments and to Alexandre Ram\'e for pointing us to an error in the original manuscript. This research was partially supported by the Israel Science Foundation (grant No. 1950/19).

\bibliography{ood_calibration}
\bibliographystyle{abbrv}

\newpage
\appendix
\renewcommand{\thefigure}{S\arabic{figure}}
\setcounter{figure}{0}
\renewcommand{\thetable}{S\arabic{table}}
\renewcommand{\thelemma}{S\arabic{lemma}}
\renewcommand{\thedefinition}{S\arabic{definition}}
\setcounter{table}{0}
\setcounter{theorem}{0}
\setcounter{definition}{0}
\setcounter{lemma}{0}

\section{Proofs for Theoretical Claims}
We begin by supplementing the definition of multiple domain calibration, extending it for the case of regression, then we provide proofs of the theorems in the paper.
\subsection{Definition of Calibration} \label{sec:calibration_intro}
Recall our definition of a calibrated classifier for binary tasks.
\begin{definition}
    Let $f:\X\rightarrow [0,1]$ and $P[X, Y]$ be a joint distribution over the features and label. Then $f(\x)$ is calibrated w.r.t to $P$ if for all $\alpha\in{[0,1]}$ in the range of $f$:
    \begin{align*}
        \E_{P}{\left[ Y \mid f(X)=\alpha\right]} = \alpha.
    \end{align*}
    In the multiple environments setting, $f(\x)$ is calibrated on $E_{\text{train}}$ if for all $e_i\in{E_{\text{train}}}$ and $\alpha$ in the range of $f$ restricted to $e_i$:
    \begin{align}
        \E{\left[ Y \mid f(X)=\alpha, E=e_i\right]} = \alpha \nonumber.
    \end{align}
\end{definition}
Let us prove the connection between multi-domain calibration and invariance, we repeat the statement of the lemma from the main paper for convenience.
\begin{lemma}[\lemref{lemma:corresp} in main paper]
    If a binary classifier $f$ is invariant w.r.t $E_{\text{train}}$ then there exists some $g:\reals\rightarrow [0,1]$ such that $g\circ f$ is calibrated on all training environments and its mean squared error on each environment does not exceed that of $f$. On the other hand, if a classifier is calibrated on all training environments it is also invariant w.r.t $E_{\text{train}}$.
\end{lemma}
\vspace{-0.5cm}
\begin{proof}
Assume that the classifier is invariant w.r.t $E_{\text{train}}$, let $e_i\in{E_{\text{train}}}$ and note that:
\begin{align*}
\E[(Y - f(X))^2 \mid E=e_i] \geq \min_{g:\reals\rightarrow\reals}{\E[(Y - g\circ f(X))^2 \mid E=e_i]}.
\end{align*}
The solution to the RHS is to take $g(\hat{\alpha}) = \E [Y \mid f(X) = \hat{\alpha}, E=e_i]$ for all $\hat{\alpha}\in{[0,1]}$ and it results in a classifier $g \circ f$ that is calibrated w.r.t $e_i$. Due to invariance, for all $\hat{\alpha}\in{\reals}$ the expectation $\E [Y \mid f(X) = \hat{\alpha}]$ is identical across all $e_i\in{E_{\text{train}}}$ where $\hat{\alpha}$ is in the range of $f$ restricted to $e_i$. Therefore there exists a single function $g$ that solves the RHS simultaneously over all environments. The resulting $g\circ f$ is indeed calibrated over all training domains and its mean squared error does not exceed that of $f$ (note that since the square loss is Bayes-consistent, this claim also holds for the classification error).
The other part of the statement that a calibrated classifier on all $E_{\text{train}}$ is invariant follows easily from the definitions.
\end{proof}

For regression tasks, one may consider a function that outputs a full CDF on $Y$ and define a calibrated classifier as one where all quantiles of the CDF match the true quantiles of $Y$ as the number of examples approached infinity. This leads to the definition in \cite{kuleshov2018accurate}, and one may follow this to analyze more general cases than the scenario we will consider in this work.

Since in this section we consider Gaussian distributions and linear regressors, a definition based on the first two moments of the distribution (instead of all quantiles of a CDF) will suffice.
Hence we will be working the following definition:
\begin{definition}
    Let $f:\X\rightarrow \reals^2$ and $P[X, Y]$ a joint distribution over the features and label. Then $f(\x)$ is calibrated w.r.t to $P$ if for all $(\alpha, \beta)\in{\reals^2}$ in the range of $f$:
    \begin{align*}
        \E{\left[ Y \mid f(X)_1=\alpha\right]} = \alpha, \: \E{\left[ Y^2 \mid f(X)_2=\beta\right]} = \beta.
    \end{align*}
    In the multiple environments setting, $f(\x)$ is calibrated on $E_{\text{train}}$ if for all $e_i\in{E_{\text{train}}}$ and $(\alpha, \beta)$ in the range of $f$ restricted to $e_i$:
    \begin{align} \label{eq:calibrated_regressor_multidomain}
        \E{\left[ Y \mid f(X)=(\alpha, \beta), E=e_i\right]} = \alpha, \: \E{\left[ Y^2 \mid f(X)=(\alpha, \beta), E=e_i\right]} = \beta.
    \end{align}
\end{definition}

\subsection{Details about ECE, MMCE and Post-Processing Methods}
To evaluate calibration and optimize our models towards multi-domain calibration, we use the Expected Calibration Error (ECE) and the Maximum Mean Calibration Error (MMCE) \cite{kumar2018trainable}.

The ECE is a scalar summary of the calibration plot, used throughout the literature to assess how well calibrated is a given classifier.
\textbf{Calibration plots} \cite{degroot1983comparison} are a visual representation of model calibration in the case of binary labels. Each example $\x$ is placed into one of $B$ bins that partition the $[0,1]$ interval, in which the output, or \emph{confidence}, of the classifier $f(\x)$ falls. For each bin $b$, the accuracy of $f$ on the bin's examples $acc(b)$ is calculated along with the average confidence $conf(b)$. These are plotted against each other to form a curve, where deviations from a diagonal represent miscalibration.\\
\textbf{ECE score} summarizes the calibration curve by averaging the deviation between accuracy and confidence:
\begin{equation} \label{eq:ece}
    ECE = \sum^B_{b=1} \frac{n_b}{N} |acc(b) - conf(b)|.
\end{equation}
$n_b$ is the number of examples in bin $b$, $N$ is the total number of examples. In all of our experiments we used $B=10$ bins of equal size.

To handle the miscalibration that is often observed in models such as neural networks \cite{guo2017calibration}, the MMCE was proposed in \cite{kumar2018trainable} as a method to improve calibration at training time. Recalling the definition of this loss: We consider a dataset $D = \{\x_i, y_i\}_{i=1}^{m}$, a binary classifier parameterized by a vector $\theta$ which we denote $f_\theta:\rightarrow [0, 1]$. The confidence of $f_\theta$ on the $i$-th example is $f_{\theta;i}=\max\{f_\theta(x_i), 1-f_\theta(x_i)\}$
and its correctness is $c_i=\mathbbm{1}_{| y_i-f_{\theta;i} | < \frac{1}{2}}$.
Then we fix a kernel $k:\reals \times \reals \rightarrow \reals$, associated with a feature map $\phi:[0,1]\rightarrow \mathcal{H}$, and MMCE over the dataset $D$ is given by:
\begin{align} \label{eq:mmce_def}
    r^{D}_{\text{MMCE}}(f_\theta) = \frac{1}{m^2}\sum_{i,j\in{D}}{(c_i-f_{\theta;i})(c_j-f_{\theta;j})k(f_{\theta;i},f_{\theta;j})}.
\end{align}
In our experiments we use an RBF kernel $k(r, r') = \exp(-\gamma(r-r')^2)$ with $\gamma=2.5$.
\eqref{eq:mmce_def} is the finite sample approximation of the following:
\begin{align} \label{eq:mmce_population}
MMCE(f_\theta ; P[X, Y]) = \| \E_{(\x, y)\sim P}[(c-f_{\theta}(\x))\phi(f_{\theta}(\x))] \|_{\mathcal{H}}.
\end{align}
Here $c$ is the correctness of $f_{\theta}$ on $(\x, y)$ as defined for \eqref{eq:mmce_def}. Attractive properties of the MMCE include it being a proper scoring rule:
\begin{theorem*}[Adapted from Thm.~1 in \cite{kumar2018trainable}]
Let $P[X, Y]$ be a probability measure defined on the space $ (\mathcal{X} \times \{0,1\})$ such that the conditionals on the pushforward measure $P[r, c] = f_\theta \sharp P$,\footnote{we note the abuse of notation here, as $f_\theta\sharp P$ is used to denote the measure that we get by applying $f_\theta$ to $X$ to obtain $r$ and $c$ is obtained by calculating its correctness w.r.t to $Y$.} $P(r \mid c = 1)$ over $([0,1] \times \{0,1\})$, $P (r \mid c = 0)$ are Borel probability measures, and let $k$ be a universal kernel. The MMCE in \eqref{eq:mmce_population} is $0$ if and only if $f_\theta$ is calibrated w.r.t $P$.
\end{theorem*}
\corref{corr:proper_score} in the paper follows by considering $\sum_{e\in{E_{\text{train}}}}MMCE(f_\theta ; P[X, Y \mid E=e])$ and applying the theorem to each summand.
For more details on the MMCE, its derivation as an integral probability measure analogue of the ECE and its properties, we refer the reader to \cite{kumar2018trainable}.

Another popular metric for calibration in binary classification problems is the Brier score, which is simply the squared error between the predicted probability and the outcome \cite{brier1950verification}:
\begin{align*}
    BS(f) = \frac{1}{m}\sum_{i=1}^{m}{(f(\x_i) - y_i)^2}.
\end{align*}
The Isotonic Regression \cite{niculescu2005predicting} post-processing methods that we use in the paper minimize the Brier score using a monotonic post-processing function. Hence we consider a classifier $f$ and a dataset $\{\x_i, y_i\}_{i=1}^{m}$. Denote the prediction of $f$ on $\x_i$ by $f_{i}$, then isotonic regression solves:
\begin{align*}
    \min_{z: f_{i} \leq f_{j} \Rightarrow z(f_{i}) \leq z(f_{j})}{\frac{1}{m}\sum_{i=1}^{m}{(z(f_{i}) - y_i)^2}}.
\end{align*}
A motivation for using this as a post-processing calibration method is the decomposition of the Brier score to a refinement and calibration score. We may denote the set of prediction values that are obtained by $f$ across the dataset by $F = \{ f_i \mid i\in{[m]}\}$. For each such value $\tilde{f}\in{F}$ then denote $N_{\tilde{f}}=|\{i \mid f_i=\tilde{f}\}|$ as the number of points for which we obtain this prediction and $y_{\tilde{f}} = \frac{1}{N_{\tilde{f}}}\sum_{i:f_i=\tilde{f}}{y_i}$ the average outcome over them:
\begin{align*}
    BS(f) = CAL(f) + REF(f) = \frac{1}{m}\sum_{\tilde{f}\in{F}}{N_{\tilde{f}}(\tilde{f}-y_{\tilde{f}})^2} + \frac{1}{m}\sum_{\tilde{f}\in{F}}{N_{\tilde{f}}(y_{\tilde{f}}(1-y_{\tilde{f}}))}
\end{align*}
The calibration score measures how far is the average prediction value from the average outcome, while refinement gives a measure of their sharpness (i.e. it raises the score of uncertain prediction). Due to the monotonicity constraint of isotonic Regression, it is usually thought of as not changing the $REF(f)$ too much, which means it minimizes the Brier score mainly by reducing $CAL(f)$.
In the multi-domain cases we are interested in, note that this vanilla isotonic regression does not take domains into account. In our experiments we use it simple by pooling the dataset on all environments and performing post-processing calibration on this dataset using isotonic regression. This procedure could output a classifier that is perfectly calibrated for the entire dataset, but not on single environments.

To give a simple variant that does post-processing while taking environments into account, we proposed a Robust Isotonic Regression method. The method minimizes the Brier score on the worst-case environment, thus aiming to bound the worst miscalibration on each environment. While in practice it will usually not provide perfect calibration on each environment, the method trades off the error between environments so it is better geared towards simultaneous calibration of the classifier on all domains. Formally we solve:
\begin{align}
    z^* =  \argmin_{z:f_i \leq f_j \Rightarrow z(f_i)\leq z(f_j)}\max_{e\in{E_{\text{train}}}}{\frac{1}{N_e}\sum_{i=1}^{N_e}{\left(z(f_{e,i})-y_i\right)^2}}.
\end{align}
Where $N_e$ are the number of data points in environment $e\in{E_{\text{train}}}$ and $f_{e,i}$ is the output of $f$ on point $i$ in the environment.

\subsection{Causal Graphical Models} \label{sec:scm}

In order to answer queries about unseen distributions based on data from different, observed distributions, one must make certain assumptions about the data generating processes and the relationships between the observed and unobserved distributions. One way of articulating such models of the world is by using causal graphs. In a causal graph, edges from a variable $X$ to a variable $Y$ mean that changing the value of $X$ \emph{may} change the distribution of $Y$. Causal graphs entail all statistical dependencies between variables, and we can read off such independence statements using the d-separation criterion \cite{pearl1994probabilistic}. We refer to background material to discuss how to identify and estimate causal effects with these causal graphical models in hand \cite{pearl2009causality}.

In the main paper, \figref{scm} illustrates our assumed causal graph for a general problem of distribution shift, and \figref{theoretical_cases} illustrates the assumed causal graph for causal and anti-causal simplified examples described in equations  \ref{eq:setting_a} and \ref{eq:setting_b}, respectively.
For instance according to d-separation, in distributions described by \figref{scm} it holds that $Y\indep E \mid \Xc, \Xacn$ and that in general $Y\nindep E \mid \Xac$. Furthermore, if we introduce a node $\Phi(\X)$ whose parents do not include $\Xac$, then $Y \indep E \mid \Phi(X)$ (and conversely, if $\Xac$ is a parent then the independence does not hold in general), which motivates the definition of a representation that has no spurious correlations.

Equipped with the definitions and background given in the previous sections, we now turn to the proofs of the theorems in the paper.

\subsection{Classification with Invariant Features} \label{sec:proof1}
We first consider the classification task from the main paper, where the data generating process is described in \figref{fig:scenario_a}. Recall that we are considering linear classifiers of the form $f(\x; \w, b)=\sigma(\w^\top\x + b)$. Our environments here are defined by the parameters of the multivariate Gaussian distributions that generate the spurious features $\{\mu_i, \Sigma_i\}_{i=1}^{k}$. As a first step we will derive the algebraic form of the constraints that calibration imposes on $\w$ and the parameters defining the environments. For convenience, we modify the notation from the main paper and consider a binary label where $\cY = \{-1, 1\}$ instead of $\cY = \{0, 1\}$.
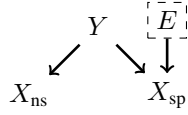
\begin{figure}[ht]
    \centering
    \begin{tikzpicture}
        \node[text centered] (y) {$Y$};
        \node[below right of = y, node distance = 1.3cm, text centered] (xac) {$X_{\text{sp}}$};
        \node[above of = xac, node distance = 1.cm, draw, rectangle, dashed, text centered] (e) {$E$};
        \node[below left of = y, node distance = 1.3cm, text centered] (xns) {$X_{\text{ns}}$};

        \draw[->, line width = 1] (y) -- (xac);
        \draw[->, line width = 1] (e) -- (xac);
        \draw[->, line width = 1] (y) to (xns);
    \end{tikzpicture}
    \caption{Diagram for data generating process in the invariant features scenario.}
    \label{fig:scenario_a}
\end{figure}
\begin{lemma} \label{lem:calibration_conditions_spurious_setting}
Assume we have $k$ environments with means and covariance matrices for environmental features $\mu_i\in{\reals^{d_e}}, \Sigma_{i}\in{\mathbb{S}_{++}
^{d_e}}, i\in{[k]}$ and a common covariance matrix $\Sigma_{\text{ns}}\in{\mathbb{S}^{d_{\text{ns}}}_{++}}$ for invariant features, where data is generated according to:
\begin{align*}
    \begin{split}
    y = \begin{cases}
    1 & \text{w.p }  \eta \\
    -1 & \text{otherwise}
    \end{cases}
    \end{split},
    \begin{split}
        \x_\text{ns} \mid Y=y \sim\N(y\mu_\text{ns}, \Sigma_{\text{ns}}), \\
        \x_{\text{sp}} \mid Y=y \sim\N(y\mu_i, \Sigma_i),
    \end{split}
\end{align*}
and $\x_{\text{ns}}, \x_{\text{sp}}$ are drawn independently. Let $\sigma:\reals\rightarrow (0, 1)$ be an invertible function and define the classifier:
\begin{align*}
    f(\x ; \w,b) = \sigma( \w^\top \x  - b).
\end{align*}
Decompose the weights $\w = [ \w_{\text{ns}}, \w_{\text{sp}} ]$ to the coefficients of the invariant and spurious features accordingly. Then if the classifier is calibrated on all environments, it holds that either $\w=\mathbf{0}$ or there exists $t \neq 0$ such that:
\begin{align} \label{eq:desideratum_lem1}
    \frac{ \w^\top_{\text{ns}}\mu_{\text{ns}} + \w^\top_{sp}\mu_i }{\w_{\text{ns}}^\top\Sigma_{\text{ns}}\w_{\text{ns}} + \w_{sp}^\top\Sigma_i\w_{sp}} = t \quad \forall i\in{[k]}.
\end{align}
\end{lemma}
\begin{proof}
Let $i\in{[k]}$, the joint distribution of features in the environment is Gaussian with mean $\hat{\mu}_i=[\mu_{\text{ns}}, \mu_i]$, covariance $ \hat{\Sigma}_i = \begin{bmatrix}
\Sigma_{\text{ns}} & 0 \\
0 & \Sigma_i
\end{bmatrix}$. Hence the output of the affine function corresponding to the classifier is a random variable with probability density function:
\begin{align*}
    P[\sigma^{-1}(f(X))=\alpha \mid Y=y, E = e_i] = (2\pi\w^\top\hat{\Sigma_i}\w)^{-\frac{1}{2}}\exp\left(\frac{\left(\alpha-y \w^\top \hat{\mu}_i + b\right)^2}{2\w^\top\hat{\Sigma}_i\w}\right).
\end{align*}
Hence the conditional probability of $Y$ is given by:
\begin{align*}
    P[Y=1 \mid \sigma^{-1}(f(X))=\alpha, E=e_i] = \frac{\eta\exp\left(\frac{\left(\alpha-\w^\top \hat{\mu}_i + b\right)^2}{2\w^\top\hat{\Sigma}_i\w}\right)}{\eta\exp\left(\frac{\left(\alpha-\w^\top \hat{\mu}_i + b\right)^2}{2\w^\top\hat{\Sigma}_i\w}\right) + (1-\eta)\exp\left(\frac{\left(\alpha+\w^\top \hat{\mu}_i + b\right)^2}{2\w^\top\hat{\Sigma}_i\w}\right)}.
\end{align*}
Note that unless $\w=\mathbf{0}$ (which results in a calibrated classifier that satisfies \eqref{eq:desideratum_lem1}), the variance of $\sigma^{-1}(f(X))$ is strictly positive since $\hat{\Sigma}_i\succ 0$, so above conditional probabilities are well-defined. Now it is easy to see that if the classifier is calibrated across environments, we need to have equality in the log-odds ratio for each $i,j$ and all $\alpha\in{\reals}$:
\begin{align*}
    \frac{\left(\alpha-\w^\top \hat{\mu}_i + b\right)^2}{2\w^\top\hat{\Sigma}_i\w} - \frac{\left(\alpha+\w^\top \hat{\mu}_i + b\right)^2}{2\w^\top\hat{\Sigma}_i\w} = \frac{\left(\alpha-\w^\top \hat{\mu}_j + b\right)^2}{2\w^\top\hat{\Sigma}_j\w} - \frac{\left(\alpha+\w^\top \hat{\mu}_j + b\right)^2}{2\w^\top\hat{\Sigma}_j\w} \quad \forall\alpha\in{\reals}.
\end{align*}
After dropping all the terms that cancel out in the subtractions we arrive at:
\begin{align*}
    \frac{\w^\top \hat{\mu}_i}{\w^\top\hat{\Sigma}_i\w} = \frac{\w^\top\hat{\mu}_j}{\w^\top\hat{\Sigma}_j\w}.
\end{align*}
This may also be written as a system of equations with an additional scalar variable $t\in{\reals}$:
\begin{align*}
    \frac{\w^\top \hat{\mu}_i}{\w^\top\hat{\Sigma}_i\w} = t \quad \forall i\in{[k]}.
\end{align*}
Now because we assumed $\Sigma_i \succ 0$ for all environments, for any solution to the above system with $t=0$, we must have:
\begin{align*}
    \w^\top\hat{\mu}_i = 0 \quad \forall i\in{[k]}.
\end{align*}
Furthermore we will have for any $\alpha\in{\reals}$:
\begin{align*}
    P[Y = 1 \mid \sigma^{-1}(f(X)) = \alpha, E=e_i] = \eta.
\end{align*}
Since we assume $f$ is calibrated and the right hand side needs to equal $\alpha$, this is only possible if $f(\x; \w, b)$ is a constant function. Again, because $\Sigma_i \succ 0$, this is only possible if $\w=\mathbf{0}$. Hence we conclude with our desired result, as can be seen by decomposing $\w$ to the parts corresponding to invariant and spurious features.
\end{proof}
We now give a result for the special case where the covariance matrices of the spurious features satisfy $\Sigma_i = \sigma^2_i\mathbf{I}$, considered in \cite{rosenfeld2020risks}. The nice correspondence here is that we will see that calibration demands one more environment than IRM to discard all spurious features. This matches the intuition that each environment reduces a degree of freedom from the set of invariant classifiers, while risk minimization reduces one more degree of freedom.
\begin{lemma} \label{thm:simple_case}
Assume we have $k \geq d_{\text{sp}}+2$ environments and define $M\left(\{\mu_i, \sigma_i\}_{i=1}^{k}\right)\in{\reals^{k\times d_e+2}}$:
\begin{align*}
    M(\{\mu_i, \sigma_i\}_{i=1}^{k}) = \begin{bmatrix}
        \mu^\top_1 & \sigma_1^2 & 1 \\
        & \vdots & \\
        \mu^\top_k & \sigma_k^2 & 1
    \end{bmatrix}.
\end{align*}

If the matrix has full rank, then for any invariant predictor the linear coefficients on spurious features are zero.
\end{lemma}
\begin{proof}
According to \lemref{lem:calibration_conditions_spurious_setting}, writing down the conditional probability $P[Y \mid \sigma^{-1}(f(\x)), E = e ]$ and demanding calibration results in the constraint that either $\w=\mathbf{0}$, and then the linear coefficients on spurious features are indeed $0$; or that for some $t\neq 0$:
\begin{align*}
    \frac{ \w^\top_{\text{ns}}\mu_{\text{ns}} + \w^\top_{\text{sp}}\mu_i }{\w_{\text{ns}}^\top\Sigma_{\text{ns}}\w_{\text{ns}} + \sigma^2_i\|\w_{\text{sp}}\|^2_2} = t \quad \forall i\in{[k]}.
\end{align*}
Without loss of generality we can phrase these constraints as:
\begin{align*}
    \frac{ \w^\top_{\text{ns}}\mu_{\text{ns}} + \w^\top_{\text{sp}}\mu_i }{\w_{\text{ns}}^\top\Sigma_{\text{ns}}\w_{\text{ns}} + \sigma^2_i\|\w_{\text{sp}}\|^2_2} = 1 \quad \forall i\in{[k]}.
\end{align*}
This is true since if $\w$ is a solution to this system of equations where the right hand side is some $t\in{\reals}$ then $t\w$ is a solution to the system where $t$ is replaced by $1$.
Rewrite the constraints again to isolate the parts depending on $\w_{\text{sp}}$:
\begin{align*}
    \sigma_i^2\|\w_{\text{sp}}\|_2^2 - \mu^\top_i \w_{\text{sp}} = \w_{\text{ns}}^{\top}\Sigma_{\text{ns}}\w_{\text{ns}} - \w_{\text{ns}}^\top\mu_{\text{ns}} \quad \forall i\in{[k]}.
\end{align*}
To find whether this system has a solution where $\w_{\text{sp}}$ is non-zero we can replace the right hand side with a scalar variable $t\in{\reals}$, and ask whether the following system has a non-zero solution:
\begin{align*}
    \sigma_i^2\|\w_{\text{sp}}\|_2^2 - \mu^\top_i \w_{\text{sp}} = t \quad \forall i\in{[k]}.
\end{align*}
For the above equations to have a non-zero solution, the following linear system must also have such a solution:
\begin{align*}
    M(\{\mu_i, \sigma_i\}_{i=1}^{k})\x = \mathbf{0}.
\end{align*}
But from our non-degeneracy condition, such a solution does not exist.
\end{proof}

Next we generalize the above to prove the result from the main paper, namely when the matrices $\{\Sigma_i\}_{i=1}^{k}$ are not diagonal.
For this purpose we introduce a definition of general position for environments, similar to the one given in \cite{arjovsky2019invariant}.
\begin{definition}
Given $k > 2\dsp$ environments with mean parameters $\{\Sigma_i, \mu_i\}_{i=1}^{k}$, we say they are in general position if for all non-zero $\x\in{\reals^\dsp}$:
\begin{align*}
    \mathrm{dim}\left(\mathrm{span}\left\{\begin{bmatrix} \Sigma_i\x + \mu_i \\
    1 \end{bmatrix}\right\}_{i\in{[k]}}\right) = d_e+1.
\end{align*}
\end{definition}
Equipped with this notion of general position, we now need to show that if it holds then the only predictors that satisfy the conditions of \lemref{lem:calibration_conditions_spurious_setting} are those with $\w_{\text{sp}}=\mathbf{0}$. Another claim we will need to prove is that the subset of environments which do not lie in general position have measure zero in the set of all possible environment settings. Hence generic environments are expected to lie in general position. This argument will follow the lines of the one given in \cite{arjovsky2019invariant}, adapted to our case with the fixed coordinate $1$ added in the above definition.
\begin{theorem}
Under the setting of \lemref{lem:calibration_conditions_spurious_setting}, if the environments lie in general position then all classifiers that are calibrated across environments satisfy $\w_{\text{sp}}=\mathbf{0}$.
\end{theorem}
\begin{proof}
According to \lemref{lem:calibration_conditions_spurious_setting}, if the predictor is calibrated then \eqref{eq:desideratum_lem1} must hold.
Following the same arguments laid out in the proof at the main paper, we get that $\w_{sp}$ needs to be a solution for the following system of equations:
\begin{align} \label{eq:ellipsoid_system}
    \w_{sp}^\top\Sigma_i\w_{sp} - \mu_i^\top\w_{sp} - t = 0 \quad \forall i\in{[k]}.
\end{align}
Now, let $\w_{sp}\in{\reals^{d_{\text{sp}}}}$ be a non-zero vector and let us define the $k\times d_e+1$ matrix:
\begin{align*}
    M(\{\mu_i, \Sigma_i\}_{i=1}^{k}, \w_{sp}) = \begin{bmatrix}
    \w_{sp}^\top\Sigma_1 - \mu^\top_1 & 1 \\
    \vdots \\
    \w_{sp}^\top\Sigma_k - \mu^\top_k & 1
    \end{bmatrix}
\end{align*}
If the environments are in general position, the above matrix has full rank for any non-zero $\w_{sp}$. Similarly to the proof of \lemref{thm:simple_case}, if \eqref{eq:ellipsoid_system} has a non-zero solution then the following system must also have a solution:
\begin{align*}
    M(\{\mu_i, \Sigma_i\}_{i=1}^{k}, \w_{sp}) \x = \mathbf{0}.
\end{align*}
Which is of course impossible due to $M(\{\mu_i, \Sigma_i\}_{i=1}^{k}, \w_{sp})$ having full rank.
\end{proof}

We conclude with the statement about the measure of sets of environments which do not lie in general position, this will follow the lines of \cite{arjovsky2019invariant}.
\begin{lemma}
    Let $k > 2d_{\text{sp}}$ and $\{\mu_{i}\}_{i=1}^{k}$ be arbitrary fixed vectors, then the set of matrices $\{\Sigma_i\}_{i=1}^{k}\in (\mathbb{S}^{d_{\text{sp}}}_{++})^{k}$ for which $\{ \Sigma_i, \mu_i \}_{i=1}^{k}$ do not lie in general position has measure zero within the set $(\mathbb{S}^{d_{\text{sp}}}_{++})^{k}$.
\end{lemma}

\begin{proof}
We assume $k > 2\dsp$ and denote by $LR(k, \dsp, r)$ the matrices of dimensions $k\times \dsp$ and rank $r$. Also for any $d$ denote by $\mathbf{1}_d$ the vector in $\reals^d$ where all entries equal $1$. Define $\M^{1}_{*}(k, \dsp)$ as the set of $k\times \dsp$ matrices of full column-rank whose columns span the vector of ones $\mathbf{1}_{k}$:
\begin{align*}
    \M^1_{*}(k, \dsp) = \{A\in{LR(k, \dsp, \dsp)} \mid \mathbf{1}_{k} \in{\mathrm{colsp}(A)} \}.
\end{align*}
Let $\{ \Sigma_i \}_{i=1}^{k}\in{(\mathbb{S}^{d_{\text{sp}}}_{++})^{k}}$ and define $\W\subseteq \reals^{k \times d_{sp}}$ as the image of the mapping $G:\reals^{d_{sp}}\setminus{\{0\}}\rightarrow \reals^{k \times d_{sp}}$:
    \begin{align*}
        (G(\x))_{i, l} = \left( \Sigma_i\x - \mu_i \right)_l
    \end{align*}
By the definition of general position given in the paper, the environments defined by $\{\Sigma_{i}, \mu_i\}_{i=1}^{k}$ lie in general position if $\W$ does not intersect $LR(k, \dsp, r)$ for all $r<\dsp$ and $\M^1_{*}(k, \dsp)$. We would like to show that this happens for all but a measure zero of $\left(\mathbb{S}^{d_{sp}}_{++}\right)^k$.

Due to the exact same arguments in Thoerem 10 of \cite{arjovsky2019invariant}, we have that $\W$ is transversal to any submanifold of $\reals^{k\times d_{\text{sp}}}$ and also does not intersect $LR(k, \dsp, r)$ where $r<\dsp$, for all $\{\Sigma_i \}_{i=1}^{k}$ but a measure zero of $\left(\mathbb{S}^{d_{sp}}_{++}\right)^k$.

It is left to show that it also does not intersect $\M^1_{*}(k, \dsp)$ for all but a measure zero of $\left(\mathbb{S}^{d_{sp}}_{++}\right)^k$. Because $\M^1_{*}(k, \dsp)$ is a submanifold of $\reals^{k\times \dsp}$, it intersects transversally with $\W$ for generic $\{\Sigma_i \}_{i=1}^{k}$. Then by transversality they cannot intersect if $\text{dim}(\W) + \text{dim}(\M^1_{*}(k, d_{\text{sp}})) - \text{dim}(\reals^{k\times d_{\text{sp}}}) < 0$. We will claim that $\text{dim}(\M^1_{*}(k, \dsp)) = k(d_{\text{sp}} - 1) + d_{\text{sp}}$ and then since $k>2d_{\text{sp}}$ we may obtain:
\begin{align*}
\text{dim}(\W) + \text{dim}(\M^1_{*}(k, d_{\text{sp}})) - \text{dim}(\reals^{k \times d_{\text{sp}}})
& \leq d_{\text{sp}} + k(d_{\text{sp}}-1) + d_{\text{sp}} - kd_{\text{sp}} \\
& = 2d_{\text{sp}} - k \\
&< 0.
\end{align*}
The negativity of the dimension implies that if $\W$ and $\M^1_{*}(k, d_{\text{sp}})$ are transversal then they do not intersect, and we may conclude our desired result that the environments lie in general position for all but a measure zero of $\left(\mathbb{S}^{d_{\text{sp}}}_{++}\right)^k$.

To show that $\text{dim}(\M^1_{*}(k, d_{\text{sp}})) = k(\dsp - 1) + \dsp$, consider a matrix $A\in{\M^1_{*}(k, d_{\text{sp}})}$. Since it has full rank, it has a $\dsp\times\dsp$ minor that is invertible. Assume this minor is just the first $\dsp$ rows of $A$, otherwise there is a linear isomorphism that transforms it into such a matrix and the arguments that follow still apply (see \cite{lee2013smooth}, Example 5.30; our proof follows a similar line of reasoning). Now write $A$ as a block matrix using $B\in{\reals^{\dsp\times\dsp}}, C\in{\reals^{(k-\dsp)\times\dsp}}$:
\begin{align*}
    A = \begin{bmatrix}
    B \\
    C
    \end{bmatrix}.
\end{align*}
Denoting by $\U$ the set of $k\times\dsp$ matrices whose first $\dsp$ rows are invertible, we consider the mapping $F:\U \rightarrow \reals^{k-\dsp}$:
\begin{align*}
    F(A) = \mathbf{1}_{k-\dsp} - CB^{-1}\mathbf{1}_{\dsp}.
\end{align*}
Clearly $F^{-1}(\mathbf{0}) = \M^1_{*}(k, d_{\text{sp}})$ and $F$ is smooth. We will show that it is a submersion by observing that its differential $DF(U)$ is surjective for each $U\in{\U}$. To this end, for a given $U=\begin{bmatrix} B \\ C \end{bmatrix}$ and any $X\in{\reals^{(k-\dsp)\times\dsp}}$ define a curve $\gamma:(-\epsilon, \epsilon) \rightarrow \U$ by:
\begin{align*}
    \gamma(t) = \begin{bmatrix}
    B \\
    C + \gamma X
    \end{bmatrix}.
\end{align*}
We have that:
\begin{align*}
    (F\circ\gamma)'(t) = \frac{d}{dt}|_{t=0}(\mathbf{1}_{k-\dsp} - (C+tX)B^{-1}\mathbf{1}_{\dsp}) = XB^{-1}\mathbf{1}_{\dsp}.
\end{align*}
Since $B^{-1}\mathbf{1}_{\dsp}$ is not the zero vector, and $X\in{\reals^{(k-\dsp)\times\dsp}}$ where $k-\dsp>\dsp$, then it is clear that the above mapping is surjective. Note that the derivatives along the curve are just a subset of the range of $DF(U)$, hence $DF(U)$ is also surjective at each point $U\in{\U}$. It follows from the submersion theorem that $\mathrm{dim}(\M^1_{*}(k, \dsp)) = k\dsp - (k-\dsp) = k(\dsp - 1) + \dsp$ as desired for our result to hold.
\end{proof}

\subsection{Regression Under Covariate Shift and Spurious Features} \label{sec:proof2}
We now move on to the second scenario presented in the paper where the mechanism $P(Y \mid X)$ is invariant and the diagram depicting the data generating process is given in \figref{fig:scenario_b}. Here for each environment $i\in{[k]}$ we will have:
\begin{align} \label{eq:regression_sem}
    &X_{c} \sim  \N(\mu^c_i, \Sigma^c_{i}) \\
    &Y = {\w^*_c}^\top \x_c + \xi, \: \xi\sim\N(0, \sigma^2_y) \nonumber\\
    &X_{sp} = y\mu_i + \eta, \: \eta\sim\N(\mathbf{0}, \Sigma_i). \nonumber
\end{align}
We consider a regressor $f:\cX\rightarrow\reals^2$, where the estimate of the mean is linear, i.e. $[f(\x; \w)]_1 = \w^\top\x$, and the estimate of the variance is constant $[f(\x; \w)]_2=c$.\footnote{Limiting the variance estimate to a constant does not make a difference for the purpose of our proof. The proof does not rely on the correctness of the variance estimate as imposed by \eqref{eq:calibrated_regressor_multidomain}, but only on the variances being equal across environments when conditioned on $f(\x)$. In other words it relies on the correctness of the mean estimate, and the distribution of $Y$ conditioned on $f(X)$ being the same across environments.}
We decompose the weights $\w$ into their parts corresponding to causal and spurious features $[\w_c, \w_{sp}]$. Then our result regarding calibration and generalization to $\cE$ is given below.
\begin{figure}[ht]
    \centering
    \begin{tikzpicture}
    \node[draw, rectangle, dashed, text centered] (e) {$E$};
        \node[below of = e, node distance=1.cm, text centered] (y) {$Y$};
        \node[right of = y, node distance=1.5cm, text centered] (xac) {$X_{\text{ac-sp}}$};
        \node[left of = y, node distance=1.2cm, text centered] (xc) {$X_{\text{c}}$};
        \draw[->, line width = 1] (e) -- (xc);
        \draw[->, line width = 1] (xc) -- (y);
        \draw[->, line width = 1] (y) -- (xac);
        \draw[->, line width = 1] (e) -- (xac);
    \end{tikzpicture}
    \caption{Diagram for data generating process in the covariate shift scenario.}
    \label{fig:scenario_b}
\end{figure}
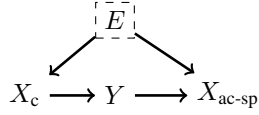
\begin{theorem} \label{thm:causal_regression}
Denote the dimensions of $X_c, X_{sp}$ by $d_c, d_{sp}$ accordingly. Assume we have $k$ environments with parameters $\{ \mu^c_i, \mu_i, \Sigma^c_i, \Sigma_i \}_{i=1}^{k}$. For any matrix $A$ denote its $i$-th row by $A^i$, and define the matrices $M(\{\mu^c_i,\mu_i\}_{i=1}^{k})\in{\reals^{k\times d_{\text{c}}+\dsp+1}}$ and $M_2(\{\mu^c_i,\Sigma^c_i\}_{i=1}^{k}, \sigma^2_y, \w_c^*)\in{\reals^{k\times d_c+2}}$ whose rows are given by:
\begin{align*}
    M(\{\mu^c_i,\mu_i\}_{i=1}^{k}) &= \begin{bmatrix}
            {\mu^c_i}^\top & \left({\w^*_c}^\top\mu^c_1\right)\mu_1^\top & 1 \\
            & \vdots & \\
            {\mu^c_k}^\top & \left({\w^*_c}^\top\mu^c_k\right)\mu_k^\top & 1
    \end{bmatrix}, \\ M_2(\{\mu^c_i,\Sigma^c_i\}_{i=1}^{k}, \sigma^2_y, \w_c^*) &= \begin{bmatrix}
            {\w_c^*}^\top\Sigma^c_1 + \left(\frac{  {\w_c^*}^\top\Sigma^c_1{\w_c^*} + \sigma^2_y}{{\w^*_c}^\top\mu^c_1} \right){\mu^c_1}^\top & \frac{{\w_c^*}^\top\Sigma^c_1{\w_c^*}}{{\w^*_c}^\top\mu^c_1} & 1 \\
            & \vdots & \\
            {\w_c^*}^\top\Sigma^c_k + \left(\frac{  {\w_c^*}^\top\Sigma^c_k{\w_c^*} + \sigma^2_y}{{\w^*_c}^\top\mu^c_k} \right){\mu^c_k}^\top & \frac{{\w_c^*}^\top\Sigma^c_k{\w_c^*}}{{\w^*_c}^\top\mu^c_k} & 1
    \end{bmatrix}.
\end{align*}
Let $f(\x; \w)$ be a calibrated regressor, assume ${\w_c^{*}}^\top\mu_i^c \neq 0$ for all $i\in{[k]}$ and that there exists $i,j\in{[k]}$ such that $\E[Y \mid E=e_i] \neq \E[Y \mid E=e_j]$. Furthermore assume that one of the following conditions hold:
\begin{itemize}
    \item $k > \max{\{d_\text{c} + 2, d_{\text{sp}}\}}$, $M_2(\{\mu^c_i,\Sigma^c_i\}_{i=1}^{k}, \sigma^2_y, \w_c^*)$ has full rank and the means of spurious features $\{\mu_i\}_{i=1}^{k}$ span $\reals^{d_{sp}}$.
    \item $k > d_c + d_{sp} + 1$ and $M(\{\mu^c_i,\mu_i\}_{i=1}^{k})$ has full rank.
\end{itemize}
then the weights of $f$ must be $\w = [\w_c^*, \mathbf{0}]$.
\end{theorem}
It is rather clear that rank-deficiency of $M_2$ would impose some highly non-trivial conditions on the relationships between $\mu_i^c, {\w_c^*}^\top\Sigma^c_i$ and the conditions given above are satisfied for all settings of environments other than a measure zero under any absolutely continuous measure on the parameters $\w_c^*, \{\mu_i^c, \Sigma_i^c\}_{i=1}^{k}$. The proof proceeds by writing the conditional distribution of $Y$ on $f(X)$, and showing that the conditions in the theorem are the direct result of the calibration constraints. 
\begin{proof}
Since $X_c, X_{sp}, Y$ are jointly Gaussian, we can write their distribution at environment $i\in{[k]}$ as:
\begin{align*}
\begin{bmatrix}
X_c \\
X_{sp} \\
Y
\end{bmatrix} \sim \N\Bigg(&\begin{bmatrix}
\mu^c_i \\
({\w^*_c}^\top \mu^c_i)\mu_i\\
{\w^*_c}^\top \mu^c_i
\end{bmatrix}, \\ &\begin{bmatrix}
\Sigma^c_i & \Sigma^c_i\w_c^*\mu_i^\top & \Sigma^c_i\w_c^* \\
\mu_i{\w_c^*}^\top\Sigma^c_i & \left({\w_c^*}^\top\Sigma^c_i{\w_c^*}+\sigma_y^2\right)\mu_i\mu_i^\top + \Sigma_i & \left({\w_c^*}^\top\Sigma^c_i{\w_c^*}+\sigma_y^2\right)\mu_i\\
{\w_c^*}^\top\Sigma^c_i &  ({\w_c^*}^\top\Sigma^c_i{\w_c^*} + \sigma_y^2)\mu_i^\top & {\w_c^*}^\top\Sigma^c_i{\w_c^*} + \sigma_y^2
\end{bmatrix}\Bigg).
\end{align*}
The predictions $\w^\top X$ are then also normally distributed, and jointly with $Y$ this can be written as:
\begin{align*}
\begin{bmatrix}
\w^\top X \\
Y
\end{bmatrix} \sim \N\Bigg(&\begin{bmatrix}
\w_c^\top\mu^c_i + (\w_{sp}^\top\mu_i)({\w^*_c}^\top\mu^c_i) \\
{\w^*_c}^\top \mu^c_i
\end{bmatrix},
\begin{bmatrix}
        \sigma^2_{f, i} & \sigma_{f,y,i} \\
        \sigma_{f,y,i} & \sigma^2_{y, i}
\end{bmatrix}
\Bigg),
\end{align*}
where we defined the items of the covariance matrix:
\begin{align*}
\sigma_{f,i}^2 &= \w_c^\top\Sigma^c_i\w_c + 2(\w^\top_c\Sigma^c_i\w^*_{c})(\mu_i^\top\w_{sp}) + \w_{sp}^\top\left(\mu_i\mu_i^\top({\w_c^*}^\top\Sigma^c_i\w_c^* + \sigma_y^2) + \Sigma_i\right)\w_{sp}, \\
\sigma_{f,y,i} &= {\w_c^*}^\top\Sigma^c_i\w_c + ({\w^*_c}^\top\Sigma^c_i\w_c^* + \sigma^2_y)\mu_i^\top\w_{sp}, \\
\sigma_{y, i}^2 &=  {\w_c^*}^\top\Sigma^c_i{\w_c^*} + \sigma_y^2.
\end{align*}
Now we can write the mean of the conditional distribution of $Y$ on $f(X)_1=\alpha$ as:
\begin{align*}
    \E\left[ Y \mid f(X)_1=\alpha, E=e_i \right] = {\w_c^*}^\top\mu^c_i + \frac{\sigma_{f,y,i}}{\sigma^2_{f,i}}(\alpha - \w_c^\top\mu^c_i - (\w_{sp}^\top\mu_i)({\w^*_c}^\top\mu^c_i)).
\end{align*}
For each environment $i\in{[k]}$, the above is a linear function of $\alpha$. Demanding $f(X)$ to be calibrated on all environments then imposes both the slopes and intercepts to be equal across environments. Writing this for the slope, we obtain that there must exist $t\in{\reals}$ such that:
\begin{align} \label{eq:slope_invariance}
    \frac{\sigma_{f, y, i}}{\sigma^2_{f, i}} = t \quad \forall i\in{[k]}. 
\end{align}
We note that $t \neq 0$ since if it is zero then we have that $\E[Y \mid f(X)_1=\alpha, E=i]$ does not depend on $\alpha$, where calibration demands that it equals $\alpha$. This can only happen if $\w_c=\mathbf{0}$, otherwise the range of $f(\x)$ is $\reals$ because we assumed in the definition of the environments that $\Sigma_c^i\succ 0$. Furthermore, $\w_c=\mathbf{0}$ cannot be calibrated if $\E[Y \mid E=e_i]$ is not constant across environments; which is also part of the non-degeneracy constraints we required. 
Next we demand the equality of the intercepts across environments. Taking these equations and replacing \eqref{eq:slope_invariance} into each of them, we get:
\begin{align*}
    {\w_c^*}^\top\mu^c_i - t\left(\w_c^\top\mu^c_i + (\w^\top_{sp}\mu_i)({\w^*_c}^\top\mu^c_i)\right) =  {\w_c^*}^\top\mu^c_j - t\left(\w_c^\top\mu^c_j + (\w^\top_{sp}\mu_j)({\w^*_c}^\top\mu^c_j)\right) \: \forall i,j\in{[k]}.
\end{align*}
Dividing both sides by $t$ and defining $\bar{\w}_c = \frac{\w^*_c}{t} - \w_c$, we can introduce another variable $t_2\in{\reals}$ and write this as a linear system of equations in variables $\w_{sp}, \bar{\w}_c, t_2$:
\begin{align} \label{eq:means_invariance}
    \bar{\w}^\top_c\mu^c_i - \w_{sp}^\top\mu_i({\w^*_c}^\top\mu^c_i) + t_2 = 0 \quad \forall i\in{[k]}.
\end{align}
We see that given $d_c + d_{sp} + 1$ environments, then with mild conditions on their non-degeneracy (i.e. the vectors containing the environment means and an extra entry of $1$ span $\reals^{d_c+d_{sp}+1}$), the only solution to the system is $\bar{\w}_c=0, \w_{sp}=0$, proving the last part of our statement.

Moving forward to demand multiple calibration on second moments $\E[Y^2 \mid f(X)_1=\alpha, E=e_i] = \E[Y^2 \mid f(X)_1=\alpha, E=e_j]$ for all $i,j\in{[k]}$, we may write this as:
\begin{align*}
    \sigma^2_{y,i} - \frac{\sigma^2_{f,y,i}}{\sigma^2_{f,i}} = \sigma^2_{y,j} - \frac{\sigma^2_{f,y,j}}{\sigma^2_{f,j}} \quad \forall i,j\in{[k]}.
\end{align*}
Plugging \eqref{eq:slope_invariance} into the above, a simplified expression is obtained:
\begin{align*}
    \sigma^2_{y,i} - t\sigma_{f,y,i} = \sigma^2_{y,j} - t\sigma_{f,y,j} \quad \forall i,j\in{[k]}.
\end{align*}
Again we can divide by $t$ and obtain an explicit expression using $\bar{\w}_c, \w_{sp}$:
\begin{align*}
    \bar{\w}_c^\top\Sigma^c_i\w_c^* - ({\w_c^*}^\top\Sigma^c_i\w_c^* + \sigma_y)\w^\top_{sp}\mu_i = \bar{\w}_c^\top\Sigma^c_j\w_c^* - ({\w_c^*}^\top\Sigma^c_j\w_c^* + \sigma_y)\w^\top_{sp}\mu_j \quad \forall i,j\in{[k]}.
\end{align*}
Finally, we can plug in \eqref{eq:means_invariance} and introduce another variable $t_3\in{\reals}$ to turn the above equations into:
\begin{align*}
    \bar{\w}_c^\top\left( \Sigma^c_i{\w_c^*} + \left(\frac{ {\w_c^*}^\top\Sigma^c_i{\w_c^*} + \sigma^2_y}{{\w^*_c}^\top\mu^c_i}\right)\mu^c_i \right) + t_2\left(\frac{ {\w_c^*}^\top\Sigma_i\w_c^*}{{\w^*_c}^\top\mu^c_i} \right) + t_3 = 0.
\end{align*}
It is now easy to see that if $k > d_c + 2$ and $\mathbf{M}_2(\{\mu_i,\Sigma_i\}_{i=1}^{k}, \sigma^2_y, \w_c^*)$ has full rank, the only solution to these equations satisfies $\bar{\w}_c=\mathbf{0}, t_2=t_3=0$. When this is plugged into \eqref{eq:means_invariance}, we find that if $k > d_{sp}$ and the spurious means span $\reals^{d_{sp}}$ then the only possible solution is $\w_{sp}=\mathbf{0}$. Finally, $\bar{\w}_c=\mathbf{0}$ means $\w^*_c=t\w_c$, and if $f(\x)$ is calibrated then we must have $t = 1$ since otherwise its estimate of the conditional mean is incorrect. Hence our proof is concluded.
\end{proof}

We note that even though the setting we considered is restricted to causal features, anti-causal non-spurious features as those in \figref{fig:scenario_a} can also be treated (resulting in the graph given in \figref{scm}). This is since for a single environment, the distribution $P[X_{\text{c}}, X_{\text{ac-ns}}, \Xac, Y \mid E=e]$ (we shorten here to $P^e$ for convenience) can always be written as follows, treating $X_{\text{ac-ns}}$ as causal features:
\begin{align*}
    P^e[X_{\text{c}}, X_{\text{ac-ns}}, X_{\text{ac-sp}}, Y] &= P^e(X_{\text{c}}, X_{\text{ac-ns}})P^e(Y \mid X_{\text{c}}, X_{\text{ac-ns}})P^e(X_{\text{ac-sp}} \mid Y, X_{\text{ac-ns}}, X_{\text{c}}) \\
    &= P^e(X_{\text{c}}, X_{\text{ac-ns}})P^e(Y \mid X_{\text{c}}, X_{\text{ac-ns}})P^e(X_{\text{ac-sp}} \mid X_{\text{ac-ns}}).
\end{align*}
The last equality is due to the separation properties of the graph, and since the joint distribution is a multivariate Gaussian, so are all the factors in the above product. Hence each environment can be described using a structural equation model of the same type as \eqref{eq:regression_sem} and \thmref{thm:causal_regression} applies.




\section{Dataset Statistics and Models} \label{sec:datasets_models}

For each of the four WILDS experiments presented in \secref{sec:exp}, we briefly describe the data and report the splits we use for training, validation and test. In each experiment we train a model on the training set, and the calibrators on the validation set. The post-processing calibrators receive tuples of model predictions and labels as input, whereas fine tuning with \ours{} receives a latent representation (values of the last hidden layer for \textit{Camelyon17} and \textit{FMoW}, and average of the representation of the cls token over the last $4$ hidden layers in \textit{CivilComments}). \ours{} is trained over a Multilayer Perceptron with $3$ hidden layers, with batch size of $64$ and the Adam optimizer. We then compare all alternatives (Original, Naive Calibration, Robust Calibration and \ours{}) on the held-out test set (OOD). Whenever an In-Domain (ID) test set is available (\textit{PovertyMap} and \textit{Camelyon17}), we evaluate the model on it as well. Throughout our experiments, we measure and report the Expected Calibration Error (ECE) using $10$ bins, dividing the $[0,1]$ interval into sub-intervals of equal length. The licenses to the datasets are CC0 for \textit{Camelyon17} and \textit{CivilComments}, \textit{FMoW} is distributed under the FMoW Challenge Public License and \textit{PovertyMap} is public domain. All model training is done on an infrastructure with 4 RTX 2080 Ti GPUs.

\subsection{\textit{PovertyMap}}
\textbf{Problem Setting} \textit{PovertyMap} is a regression task of poverty mapping across countries. Input $\x$ is a multispectral satellite image, output $y$ is a real-valued asset wealth index and domain $d$ is a country and whether the satellite image is of an urban or a rural area. The goal is to generalize across countries and demonstrate subpopulation performance across urban and rural areas.

\textbf{Data} \textit{PovertyMap} is based on a dataset collected by \cite{yeh2020using}, which organized satellite images and survey data from 23 African countries between 2009 and 2016. There are 23 countries, and every location is classified as either urban or rural. Each example includes the survey year, and its urban/rural classification. 

\begin{enumerate}
    \item Training: 10000 images from 13 countries.
    \item Validation (OOD): 4000 images from 5 different countries (distinct from training and test (OOD) countries).
    \item Test (OOD): 4000 images from 5 different countries (distinct from training and validation (OOD) countries).
    \item Validation (ID): 1000 images from the same 13 countries in the training set.
    \item Test (ID): 1000 images from the same 13 countries in the training set.
\end{enumerate}

\subsection{\textit{Camelyon17}}
\textbf{Problem Setting} \textit{Camelyon17} is a tumor identification task across different hospitals. Input $\x$ is an histopathological image, label $y$ is a binary indicator of whether the central region contains any tumor tissue and domain $d$ is an integer identifying the hospital. The training and validation sets include the same four hospitals, and the goal is to generalize to an unseen fifth hospital. We note that in \cite{koh2020wilds} they include data from three hospitals in the training set and validate on data from a fourth hospital. Our setting includes a validation set from multiple hospitals since our fine tuning methods requires multiple domains.

\textbf{Data} The dataset comprises 450000 patches extracted from 50 whole-slide images (WSIs) of breast cancer metastases in lymph node sections, with 10 WSIs from each of five hospitals in the Netherlands \cite{bandi2018detection}. Each WSI was manually annotated with tumor regions by pathologists, and the resulting segmentation masks were used to determine the labels for each patch. Data is split according to the hospital from which patches were taken.

\begin{enumerate}
    \item Training: 335996 patches taken from each of the 4 hospitals in the training set.
    \item Validation: 60000 patches taken from each of the 4 hospitals in the training set (15000 patches from each hospital).
    \item Test (OOD): 85054 patches taken from the 5th hospital, which was chosen because its patches were the most visually distinctive. 
\end{enumerate}

\subsection{\textit{CivilComments}}
\textbf{Problem Setting} \textit{CivilComments} is a toxicity classification task across different demographic identities. Input $\x$ is a comment on an online article, label $y$ indicates if it is toxic, and domain $d$ is a one-hot vector with 8 dimensions corresponding to whether the comment mentions either of the 8 demographic identities \textit{male}, \textit{female}, \textit{LGBTQ}, \textit{Christian}, \textit{Muslim}, \textit{other religions}, \textit{Black}, and \textit{White}. The goal is to do well across all subpopulations, as computed through the average and worst case model performance.

\textbf{Data} \textit{CivilComments} comprises 450000 comments, annotated for toxicity and demographic mentions by multiple crowdworkers, where toxicity classification is modeled as a binary task \cite{borkan2019nuanced}. Each comment was originally made on an online article. Articles are randomly partitioned into disjoint training, validation, and test splits, and then formed the corresponding datasets by taking all comments on the articles in those splits. 

\begin{enumerate}
    \item Training: 269038 comments.
    \item Validation: 45180 comments.
    \item Test: 133782 comments.
\end{enumerate}

\subsection{\textit{FMoW}}
\textbf{Problem Setting} \textit{FMoW} is a building and land  multi-class classification task across regions and years. Input $\x$ is an RGB satellite image, label $y$ is one of 62 building or land use categories, and domain $d$ is the time the image was taken and the geographical region it captures. The goal is to generalize across time, and improve subpopulation performance across all regions. 

\textbf{Data} \textit{FMoW} is based on the Functional Map of the World dataset \cite{christie2018functional}, which includes over 1 million high-resolution satellite images from over 200 countries, based on the functional purpose of the buildings or land in the image, over the years 2002–2018. We use a subset of this data introduced in \cite{koh2020wilds}, which is split into three time range domains, 2002–2013, 2013–2016, and 2016–2018, as well as five geographical regions as subpopulations: \textit{Africa}, \textit{Americas}, \textit{Oceania}, \textit{Asia} and \textit{Europe}. 

\begin{enumerate}
    \item Training: 76863 images from the years 2002–2013.
    \item Validation (OOD): 19915 images from the years from 2013–2016.
    \item Test (OOD): 22108 images from the years from 2016–2018.
    \item Validation (ID): 11483 images from the years from 2002–2013.
    \item Test (ID): 11327 images from the years from 2002–2013.
\end{enumerate}

\paragraph{Models}
In the following we briefly describe each of the models used in the experiments reported in \secref{sec:exp}.
\begin{itemize}
    \itemsep0em
    \item \textbf{BERT} - BERT is a 12-layer Transformer model \cite{vaswani2017attention} that represents textual inputs contextually and sequentially \cite{devlin2019bert}. It is widely used in NLP, and is considered the standard benchmark for any state-of-the-art system. It was previously shown to be miscalibrated across its training and test environments \cite{desai2020calibration}. In our \textit{CivilComments} experiments, we use BERT-base-uncased, a smaller variant of BERT which has a layer size of 768
    \item \textbf{DenseNet} - Dense Convolutional Network (DenseNet), is a feed-forward neural network where for each layer, the feature-maps of all preceding layers are used as inputs, and its own feature-maps are used as inputs into all subsequent layers \cite{huang2017densely}. DenseNets are widely used in computer vision, especially for image classification tasks . We use a DenseNet-121 model, a DenseNet variant with 121 layers, in the \textit{Camelyon17} and \textit{FMoW} experiments.
    \item \textbf{ResNet} - Residual Network (ResNet) is a feed-forward neural network where layers are reformulated to learning residual functions with reference to the layer inputs \cite{he2016identity}. DenseNets where shown to be successful in multiple image recognition tasks. We use the 18-layer variant, ResNet-18, in the \textit{PovertyMap} experiment.
\end{itemize}

We run our models using the default setting used in \cite{koh2020wilds}. Each model is trained four times, using a different random seed at each run. We report performance averages and their standard deviation in \secref{sec:exp}.

\paragraph{Robustness to Model Architecture Choice}
For each of the five WILDS datasets we report results on (\textit{PovertyMap}, \textit{Camelyon17}, \textit{CivilComments} and \textit{FMoW}) also tested the robustness of our results to different model architectures. In the following we describe the architecture we tested for each dataset, and the relative results achieved.

\begin{itemize}
    \item \textbf{BERT} - We used a pre-trained \textit{BERT} in the \textit{Civilcomments} experiments. On the \textit{Civilcomments} dataset, we compared results on the BERT-base-uncased model with the cased and large versions. While we did find the performance increases with model size, perfromance drops on OOD examples remained consistent across models, with \ours{} outperforming \textit{Robust Calibration} and \textit{Naive Calibration} by an average of $1.4 \%$ and $3.1 \%$ (absolute), respectively.
    \item \textbf{DenseNet} - In the \textit{FMoW} experiments, we tested the relative performance of the $121$ layer version to the $169$ and $201$ layer alternatives available via \url{https://pytorch.org/hub/pytorch_vision_densenet/}. Differences between the three models were not statistically significant. 
    \item \textbf{ResNet} - In the \textit{PovertyMap} experiments, we compare \textit{ResNet-18} to the $34$ and $50$ layers alternatives available via \url{https://pytorch.org/hub/pytorch_vision_resnet/}. We found that \textit{ResNet-18} performs slightly on the OOD test set, with average gain of $0.01$ in pearson correlation compared with \textit{ResNet-34}. \textit{Robust Calibration} remained better than \textit{Naive Calibration} and the original model across runs. 
\end{itemize}


\paragraph{Training Algorithms}
In the WILDS experiments, for each dataset we train our models using three out of these four alternatives:
\begin{itemize}
    \itemsep0em
    \item \textbf{ERM} - Empirical risk minimization (ERM) is a training algorithms the looks for models that minimize the average training loss, regardless of the training environment.
    \item \textbf{IRM} Invariant risk minimization (IRM) \cite{arjovsky2019invariant} is a training algorithm that penalizes feature distributions that have different optimal linear classifiers for each environment.
    \item \textbf{DeepCORAL} 
    DeepCORAL \cite{sun2016deep} is an algorithm that penalizes differences in the means and covariances of the feature distributions for each training environment. It was originally proposed in the context of domain adaptation, and has been subsequently adapted for domain generalization \cite{gulrajani2020search}.
    \item \textbf{GroupDRO} - Group DRO \cite{hu2018does} uses distributionally robust optimization (DRO) to explicitly minimize the loss on the worst-case environment.
\end{itemize}

We do not perform any hyperparameter search, and use the default version available in \cite{koh2020wilds}.

\section{Experiments on Colored MNIST} \label{sec:cmnist_results}
For the colored MNIST\footnote{The MNIST dataset is available under the terms of the Creative Commons Attribution-Share Alike 3.0 license} dataset we trained Multi-Layer Perceptrons (MLPs) with ERM, IRMv1 and \ours{}, based on the code provided in \cite{kamath2021does} with the following adjustments: we add \ours{} and optimize it using SGD with batches of size $512$ from each training environment, for $5001$ steps at each run (~$50$ epochs). We used either the Adagrad optimizer \cite{duchi2011adaptive} or Adam \cite{kingma2014adam} (Adam was replaced with Adagrad in one environment where it produced highly unstable training metrics). All models were trained on a single NVidia Tesla P100 GPU virtual machine, on the Google Cloud Platform. Other algorithms were trained with Gradient Descent (i.e. without batching the dataset, which is infeasible for \ours{} since it is based on kernels) and Adam for $500$ steps/epochs, exactly as done in the code provided by \cite{arjovsky2019invariant, kamath2021does}. For \ours{}, hyperparamters are drawn similarly to the rest of the algorithms, except when using Adagrad where we multiply the originally drawn learning rate by $5$.

\subsection{Performance of \ours{}}
We will refer to environments with tuples $(\alpha, \beta)$ that denote correlation with digit and color respectively, as done in \secref{subsec:2bit}. For each setting of training and test environments we experiment with, $100$ models are trained using each algorithm: ERM, IRM and \ours{}.
To illustrate the failure case pointed out in \cite{kamath2021does} and \secref{sec:exp} of the paper, we train the algorithms with training environments corresponding to $e_1=(0.1, 0.05), e_2=(0.2, 0.05)$ and use data from test environment $e_3=(0.9, 0.05)$. \figref{fig:box_cmnist_color05} which we produce using code provided in \cite{kamath2021does} shows the results, where each point corresponds to a model trained with some set of drawn hyperparameters. Most models trained by \ours{} achieve log-loss that is close to that of the optimal invariant classifier (marked by dashed black line), while the models trained with IRMv1 are more scattered and specifically those that achieve lower log-loss are the ones that also obtain lower training objective. The bold colored lines mark the points that minimize $\sum_{e\in{E_\text{train}}}{l^e(f_\theta) + \lambda\cdot r^e(f_\theta)}$ with $\lambda=10^6$ (expect for ERM where it's the point which minimizes the empirical loss), showing that out of the models trained with IRMv1, the one which minimizes the objective has loss close to that of the solution $\text{OPT}_{\text{IRMv1}}$ from \figref{fig:mnist}(a) in the paper (marked by dashed red line). That is while the \ours{} model with the lowest training objective is very close to the optimal invariant classifier in its test loss (marked by black dashed line).
\begin{figure}[h]
    \centering
    \includegraphics[width=0.7\textwidth]{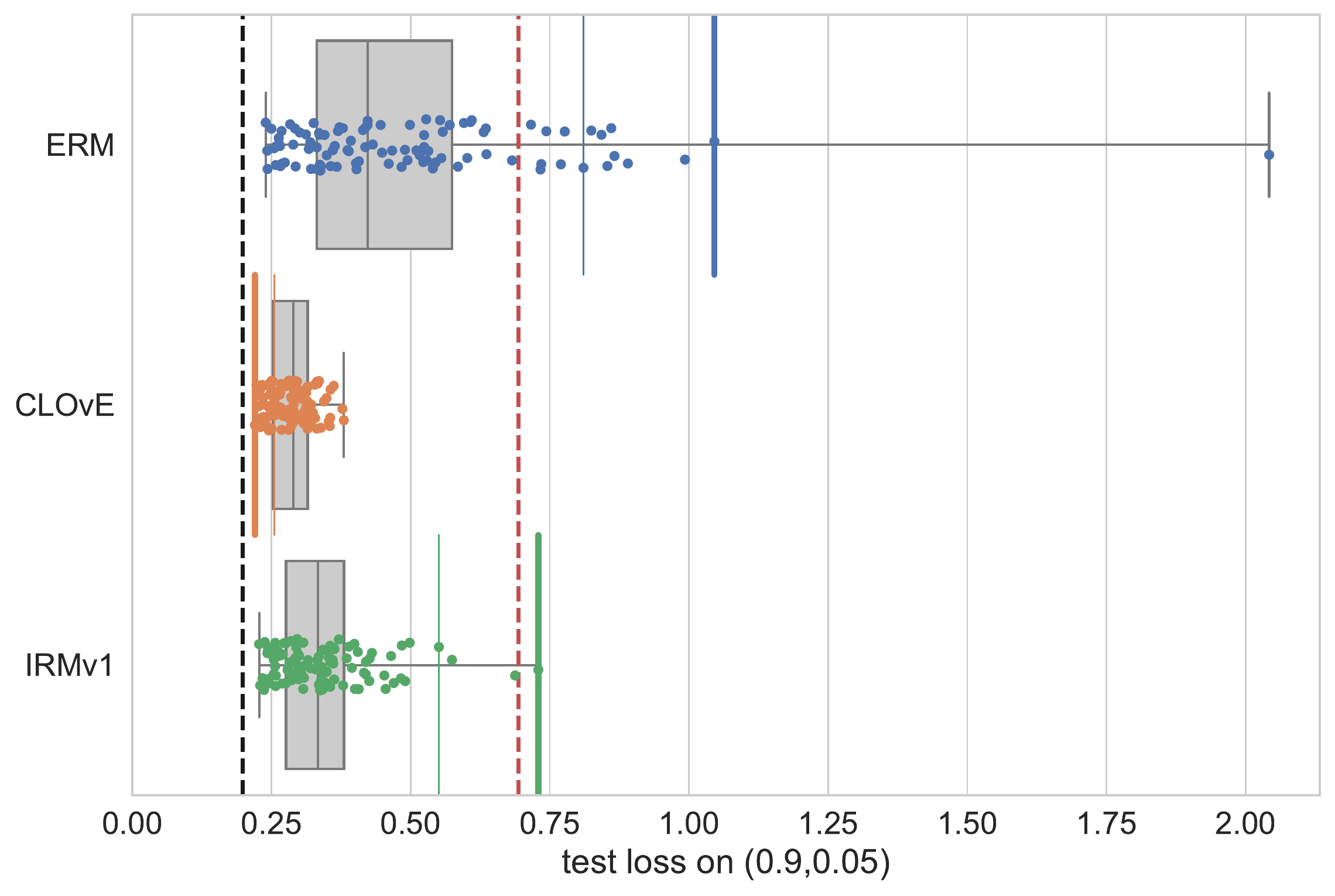}
    \caption{Log-loss on test environment $(0.9, 0.05)$ of classifiers trained with ERM, \ours{} and IRMv1 on training environments $(0.1, 0.05), (0.2, 0.05)$. Black dashed line marks the log-loss achieved by the optimal invariant classifier, while the red dashed line shows the loss achieved by $\text{OPT}_{\text{IRMv1}}$ from \figref{fig:mnist}(a). Bold colored lines mark the test loss achieved by the model which minimizes $\sum_{e\in{E_\text{train}}}{l^e(f_\theta) + \lambda\cdot r^e(f_\theta)}$ with $\lambda=10^6$ out of all trained models.
    }
    \label{fig:box_cmnist_color05}
\end{figure}
Note that in this case color is the invariant feature while the digit is spurious. For the opposite case, where the digit is invariant, the error incurred by MLPs in digit recognition makes it difficult to find the exact invariant classifier by optimizing \ours{} (since this error is close to the magnitude of the $0.05$ correlation). Yet in \secref{subsec:cmnist_model_selection} the failure case of IRMv1 in these environments will be illustrated by average ECE (which \ours{} is a surrogate for) being a better measure of invariance than the IRMv1 objective.

The experiment presented in \cite{arjovsky2019invariant} used the training environments $e_1=(0.25, 0.1), e_2=(0.25, 0.2)$ with test environments $e_3=(0.25, 0.9)$, where IRMv1 can in principle learn the optimal invariant classifier. We give the results on learning with these environments for completion. As can be observed in \figref{fig:box_cmnist_digit25}, both \ours{} and IRMv1 learn models that are close to the optimal invariant one. While IRMv1 learned more of those models during the hyperparameter sweep\footnote{This can be attributed to the choice of ranges for drawing hyperparameters which we did not carefully tune to accommodate \ours{}.}, \ours{} still obtains some close-to-invariant models during the sweep.
\begin{figure}[h]
    \centering
    \includegraphics[width=0.7\textwidth]{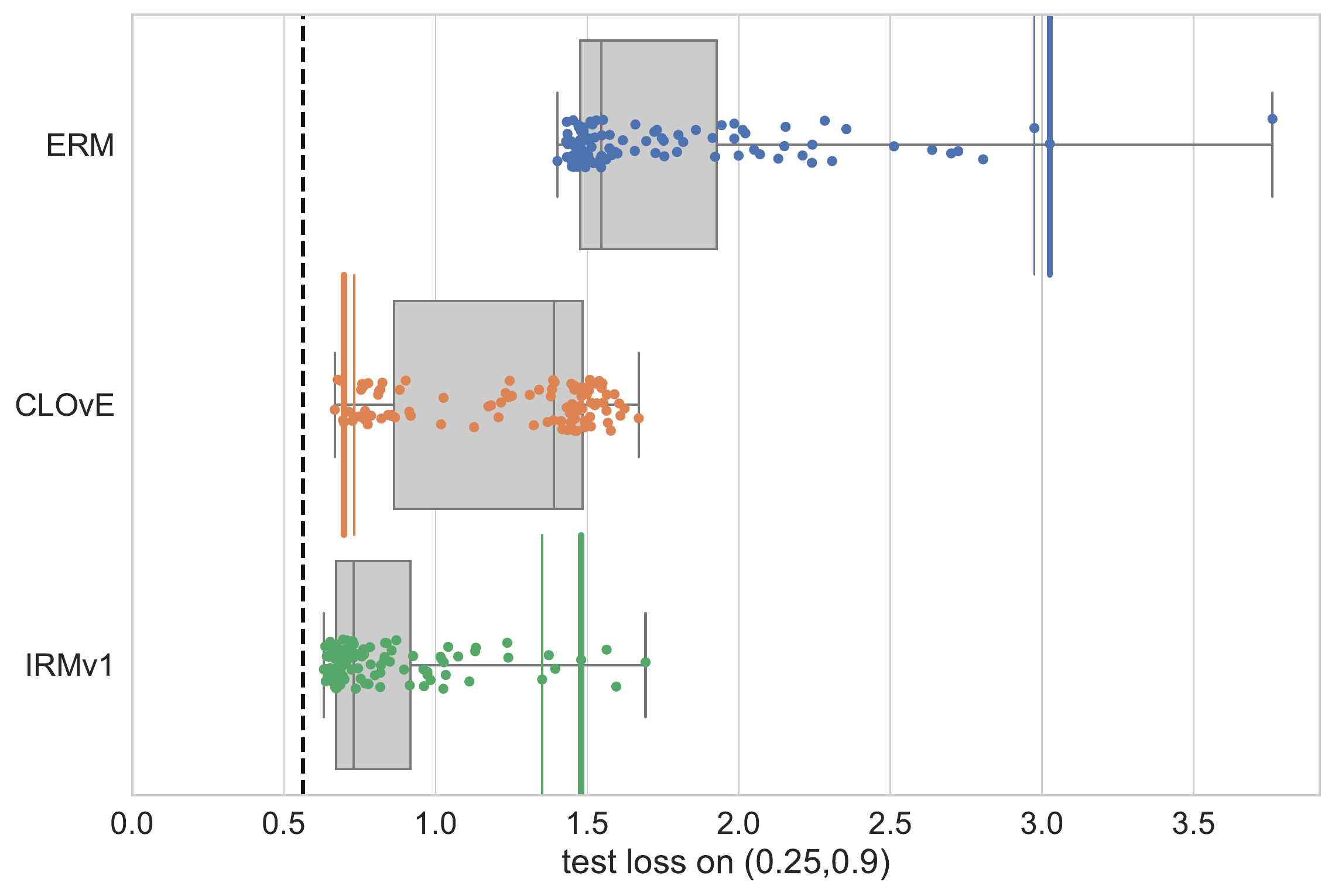}
    \caption{Log-loss on test environment $(0.25, 0.9)$ of classifiers trained with ERM, \ours{} and IRMv1 on training environments $(0.25, 0.1), (0.25, 0.2)$. Lines denote the same corresponding quantities in \figref{fig:box_cmnist_color05}, except we omit the red dashed line from that figure.
    }
    \label{fig:box_cmnist_digit25}
\end{figure}
The rest of this section will be dedicated to studying model selection with the proposed average ECE criterion and the correlation between ID average ECE and OOD performance.
\subsection{Model Selection Experiments} \label{subsec:cmnist_model_selection}
\begin{figure}[h]
    \centering
    \includegraphics[width=0.5\textwidth]{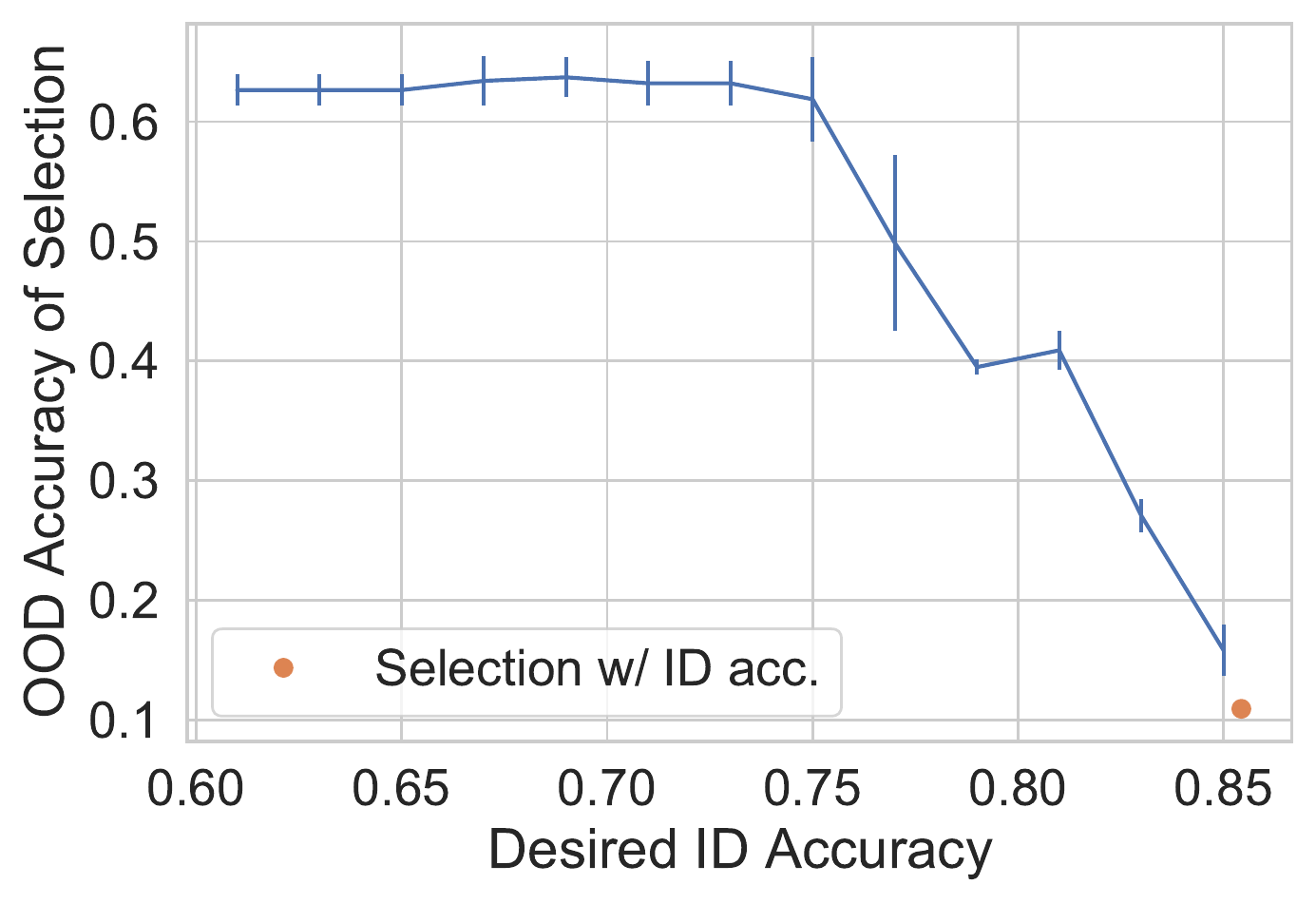}
    \caption{Model selection under a constraint on the ID accuracy. The OOD accuracy obtained by the proposed model selection method is plotted against the desired In-Domain accuracy, $\text{acc}_{\text{ID}}$, which is the minimal validation accuracy that we constrain the selected model to achieve. Red marker denotes the performance of the model achieved by selection based on ID validation accuracy alone. 
    }
    \label{fig:cmnist_selection_curve}
\end{figure}
Let us recall and elaborate the selection procedure proposed in \secref{sec:algs}:
\begin{itemize}
    \item Given a desired threshold for In-Domain accuracy $\text{Thr}_{\text{ID}}$ and a set of models $f_1(\x), \ldots, f_n(\x)$ from which we would like to select a candidate, perform the following.
    \item For each candidate model $\hat{f}$, recalibrate it with Isotonic Regression or some other preferred post-processing technique \footnote{This is a crucial step, since models that are highly miscalibrated can become well-calibrated upon post-processing}. Calculate its ID validation error $\text{val}_{\text{ID}}(\hat{f})$ over a held-out dataset. For the held-out dataset from each environment $e\in{E_{\text{train}}}$ also calculate $ECE^e(\hat{f})$: the $ECE$ of $\hat{f}$ over this dataset. Then take $ECE(\hat{f}) = \sum_{e\in{E_{\text{train}}}}{ECE^e(\hat{f})}$.
    \item Choose $\mathrm{arg}\min_{\hat{f}: \text{val}_{\text{ID}}(\hat{f}) \geq \text{Thr}_{\text{ID}}}{ECE(\hat{f})}$.
\end{itemize}
\begin{figure}[ht]
    \centering
    \subfigure[]{\centering
    \includegraphics[width=0.45\textwidth]{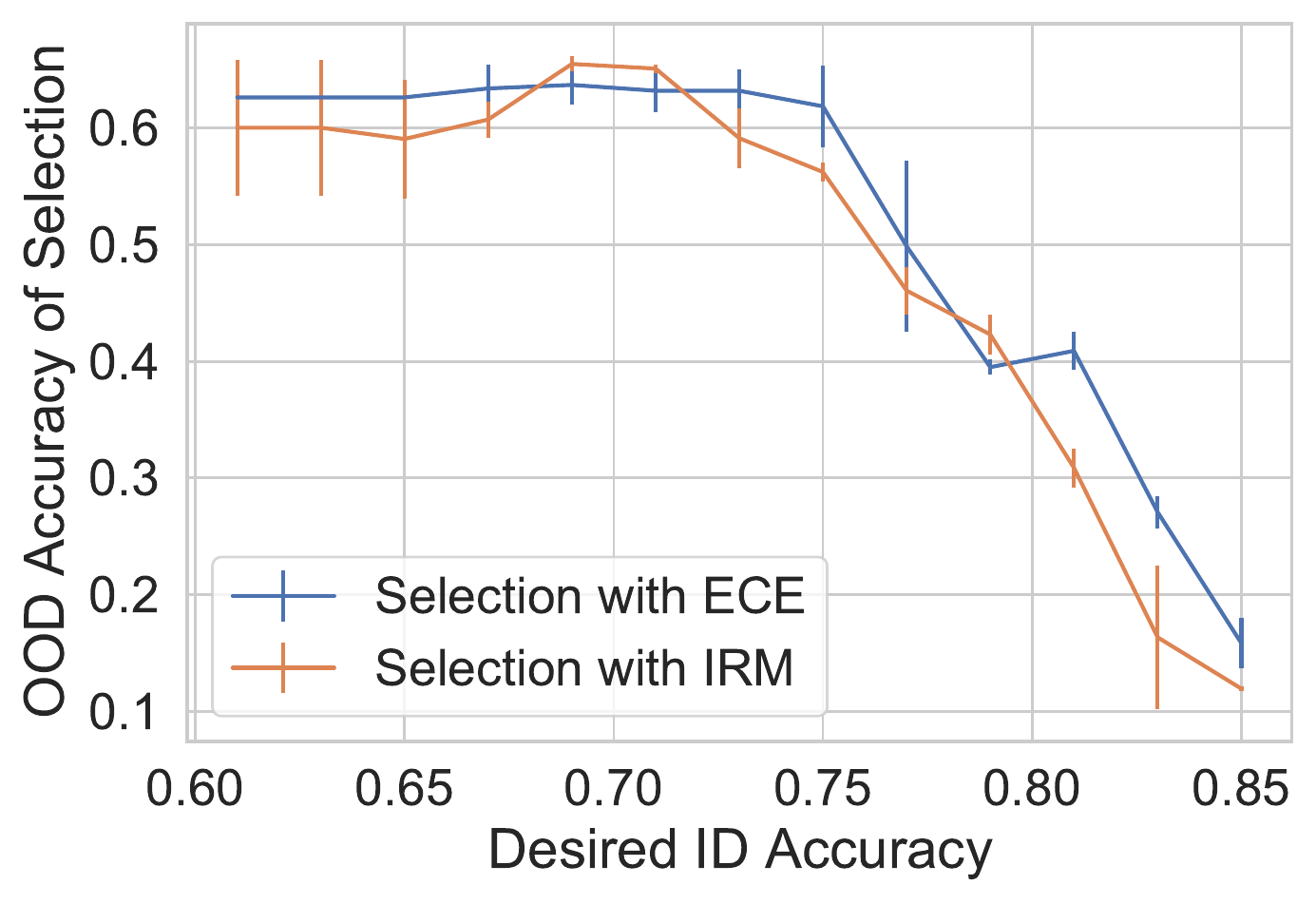}}
    \subfigure[]{\centering
    \includegraphics[width=0.45\textwidth]{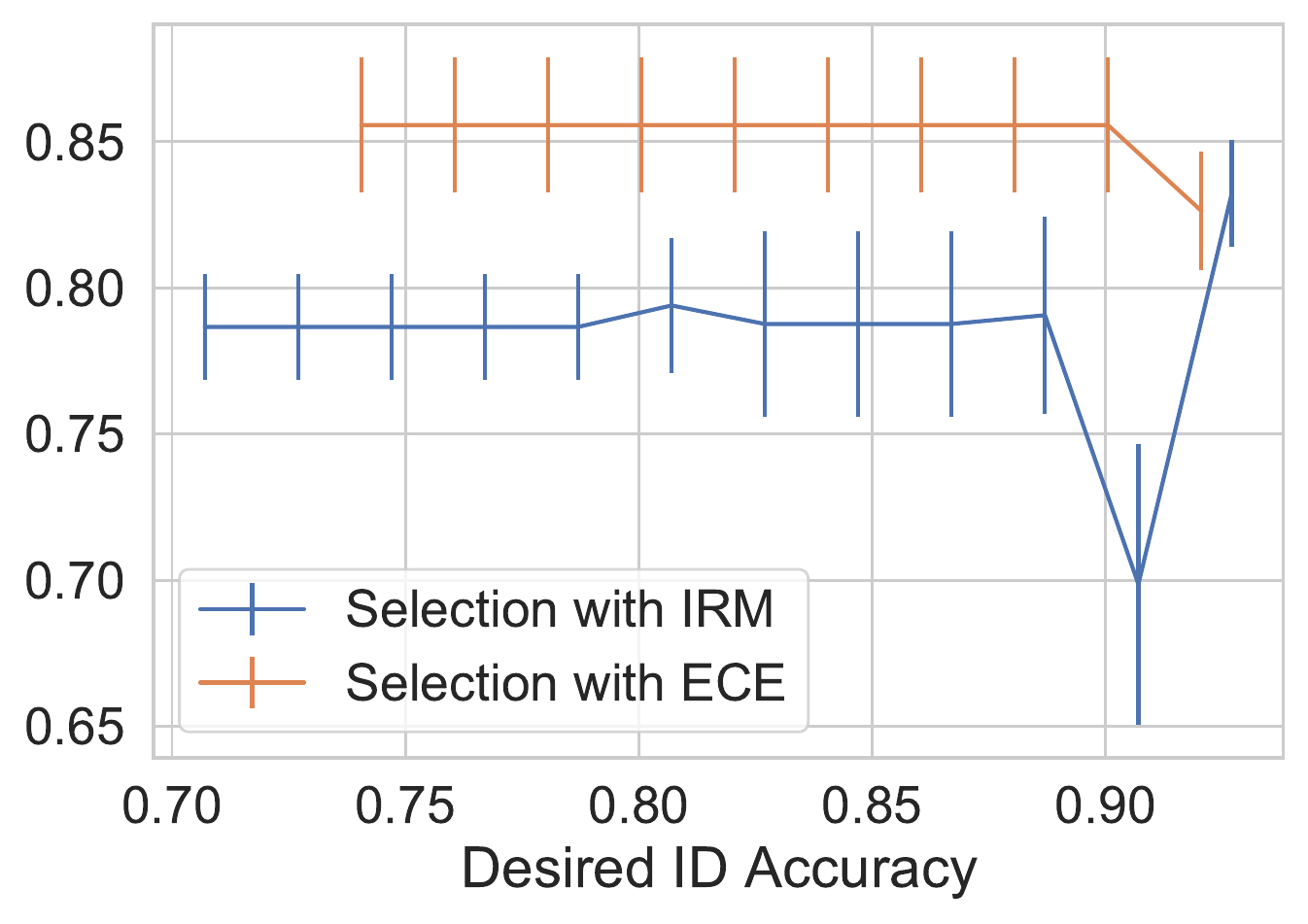}}
    \caption{Comparison of proposed model selection procedure applied with IRMv1 objective and the average ECE over training environments in two settings. (a) $e_1=(0.05, 0.1),e_2=(0.05, 0.2), e_{\text{test}}=(0.05, 0.9)$ and (b) $e_1=(0.25, 0.1),e_2=(0.25, 0.2), e_{\text{test}}=(0.25, 0.9)$.}
    \label{fig:irm_vs_ece_selection}
\end{figure}
\begin{figure}[ht]
    \centering
    \subfigure[]{\centering
    \includegraphics[width=0.8\textwidth]{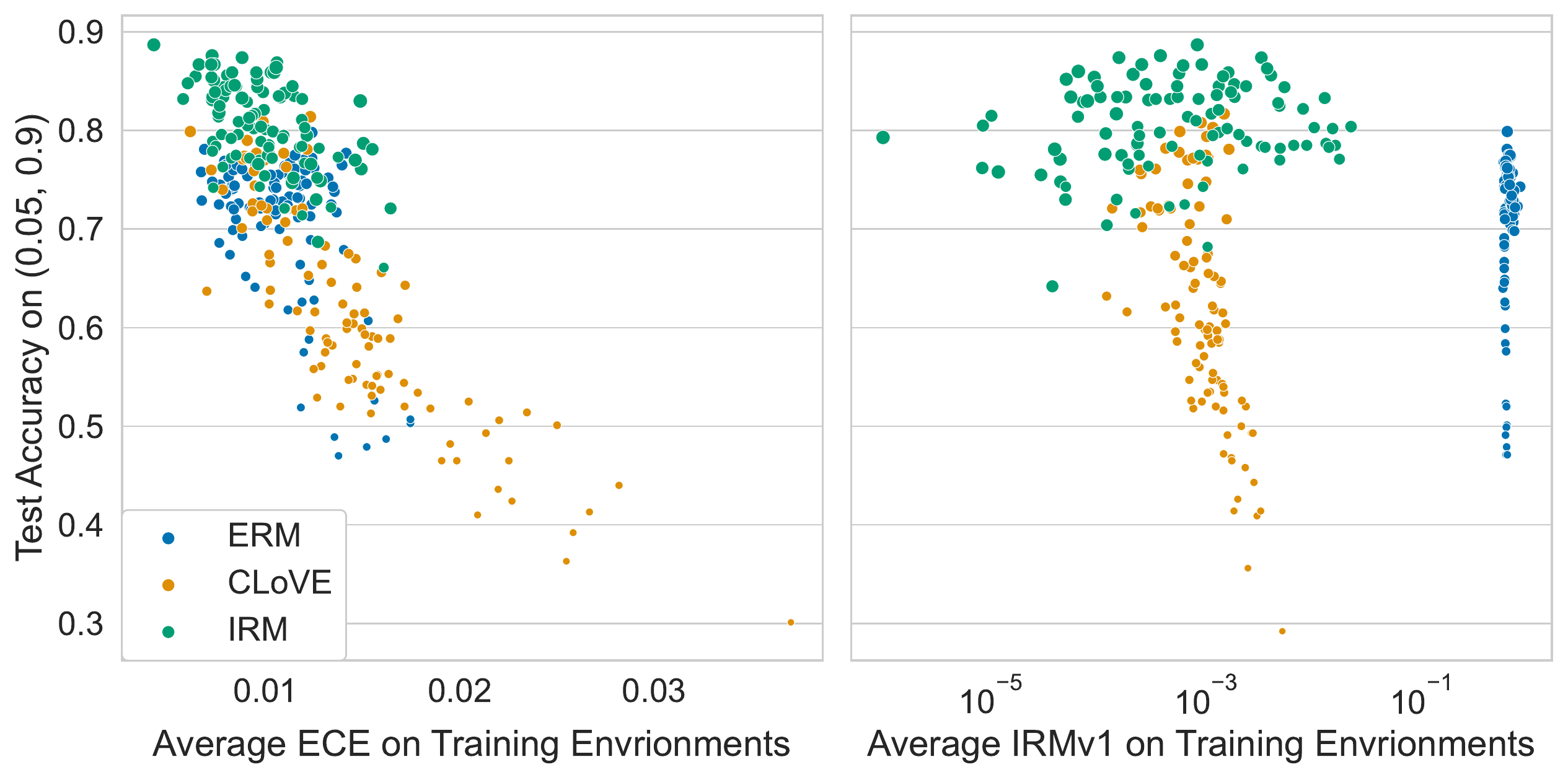}} \\
    \subfigure[]{\centering
    \includegraphics[width=0.8\textwidth]{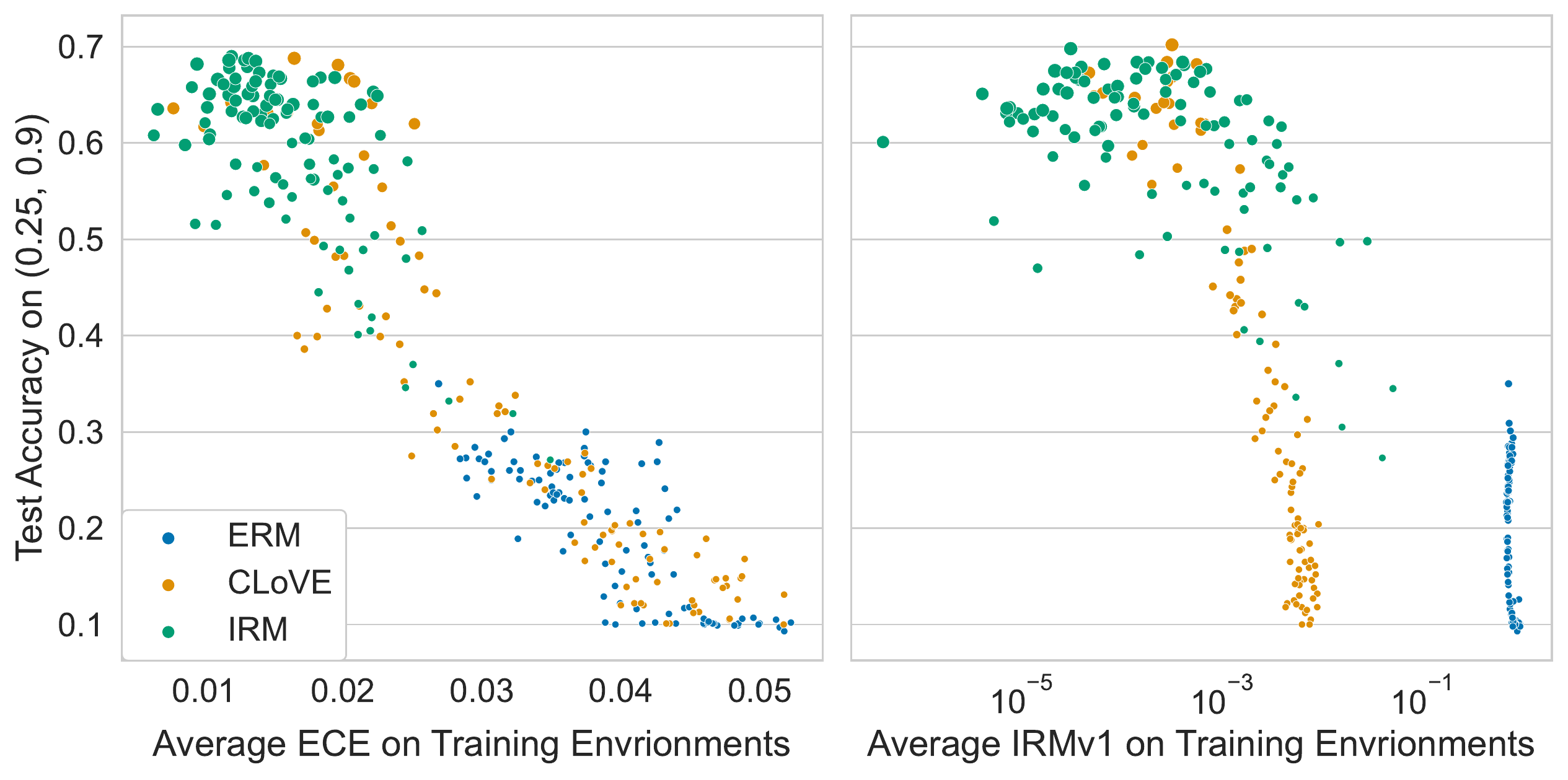}}
    \caption{Scatter plots of average ECE, and average IRMv1 penalty over training environments against the accuracy on test environments in settings (a) $e_1=(0.05, 0.1),e_2=(0.05, 0.2), e_{\text{test}}=(0.05, 0.9)$ and (b) $e_1=(0.25, 0.1),e_2=(0.25, 0.2), e_{\text{test}}=(0.25, 0.9)$. Size of marker is proportional to the ratio between OOD and ID accuracies.}
    \label{fig:ece_irm_vs_ood_expanded}
\end{figure}
\textbf{Selection with minimal ECE facilitates a tradeoff between ID accuracy and stability.}
We use the trained models from the last section (all models trained with either ERM, IRMv1 or \ours{} are pooled into a set of candidates), over environments $e_1=(0.25, 0.1), e_2=(0.25, 0.2)$. Selecting the model with minimal $\text{val}_{ID}(\hat{f})$ delivers a classifier with $10.96\%(\pm 0.81)$ accuracy on $e_{\text{test}}=(0.25, 0.9)$ and $85.43\%(\pm 0.13)$ accuracy on the training environments. The trade-off achieved by selection with the proposed criterion is shown in \figref{fig:cmnist_selection_curve}.
Demanding ID accuracy that is higher than $75\%$ (the ID error obtained by an optimal invariant classifier) yields a relatively sharp drop towards the OOD accuracy obtained by a classifier that purely minimizes empirical error. Going below $75\%$ retrieves a classifier that achieves $64.98\%(\pm 2.67)\%$ OOD accuracy.
\\ \textbf{Comparison with IRMv1 Penalty as Selection Criterion.}
As a baseline to the average ECE over training environments we compare it with using the value of the IRMv1 regularizer, also calculated with a validation set from each training environment. In \figref{fig:irm_vs_ece_selection} we compare the curves obtained by the proposed model selection procedure, and that same procedure when replacing the ECE with the value of IRMv1. \figref{fig:irm_vs_ece_selection}(a) shows the result on the scenario where $e_1=(0.25, 0.1),e_2=(0.25, 0.2)$ and $e_{\text{test}}=(0.25, 0.9)$. In this case the two methods are quite comparable, expect for the tail of high desired ID accuracies, where the chosen models are trained with ERM and the IRMv1 criterion fails to rank them by their OOD accuracy. \figref{fig:irm_vs_ece_selection}(b) shows the same plot on the scenario where $e_1=(0.05, 0.1),e_2=(0.05, 0.2)$ and $e_{\text{test}}=(0.05, 0.9)$, which corresponds to the failure case of IRM in \figref{fig:mnist}(a). Due the observation of \cite{kamath2021does}, we may expect the IRMv1 objective to fail at capturing invariance in this setting. Indeed, the model selection done using the IRMv1 penalty gives a worst model than the one selected by ECE in this case. In \figref{fig:ece_irm_vs_ood_expanded} we also plot the correspondence between OOD accuracy and these quantities (namely ID average ECE, and IRMv1 penalty) as in \figref{fig:mnist}(b) for both settings depicted in \figref{fig:irm_vs_ece_selection} showing the erratic behavior of the IRM penalty when considered on different training regimes.

\end{document}